\documentclass{article} 

\usepackage{microtype}
\usepackage{graphicx}
\usepackage[skip=1pt]{caption}
\usepackage{subcaption}
\usepackage{booktabs}
\usepackage{hyperref}
\usepackage{tcolorbox}
\tcbuselibrary{skins, breakable}

\usepackage{enumitem}
\setlist{nolistsep}
\usepackage{placeins}
\usepackage{wrapfig}
\usepackage[final]{colm2026_conference}

\usepackage{url}

\usepackage{lineno}

\definecolor{darkblue}{rgb}{0, 0, 0.5}
\hypersetup{colorlinks=true, citecolor=darkblue, linkcolor=darkblue, urlcolor=darkblue}

\usepackage{amsmath}
\usepackage{amssymb}
\usepackage{mathtools}
\usepackage{amsthm}

\theoremstyle{plain}
\newtheorem{theorem}{theorem}[section]
\newtheorem{proposition}[theorem]{Proposition}

\newtheorem{corollary}[theorem]{Corollary}
\theoremstyle{definition}

\theoremstyle{remark}

\DeclareMathOperator{\rank}{rank}

\DeclareMathOperator*{\argmax}{arg\,max}

\usepackage[group-separator={,}]{siunitx}

\usepackage{float}

\definecolor{green(html/cssgreen)}{rgb}{0.0, 0.5, 0.0}

\newcommand{\colorchar}[1]{%
  \ifnum\pdfstrcmp{#1}{A}=0
    \textcolor{green(html/cssgreen)}{\small{\texttt{#1}}}%
  \else
    \textcolor{red}{\small{\texttt{#1}}}%
  \fi
}

\usepackage[capitalize,noabbrev,nameinlink]{cleveref}
\crefname{resultboxctr}{Result}{Results}
\Crefname{resultboxctr}{Result}{Results}
\crefalias{theorem}{proposition}

\newtcolorbox[use counter=resultboxctr]{resultbox}[2][]{
    colback=gray!5,
    colframe=gray!50,
    label={#2},
    boxrule=0.4pt,
    arc=3pt,
    left=6pt, right=6pt, top=4pt, bottom=4pt,
    fontupper=\normalsize,
    colbacktitle=gray!20,
    coltitle=black,
    fonttitle=\normalsize\bfseries,
    toptitle=1pt,
    bottomtitle=1pt,
    title={Result~\thetcbcounter\ifx&#1&\else: #1\fi},
}

\title{\emph{Lost in Backpropagation}: \\ The LM Head is a Gradient Bottleneck}

\author{Nathan Godey, Yoav Artzi \\
Cornell University \\
New York, USA \\
\texttt{godeynathan@gmail.com} 
}

\begin{document}

\ifcolmsubmission
\linenumbers
\fi

\maketitle

\begin{abstract}
The last layer of neural language models (LMs) projects output features of dimension $D$ to logits in dimension $V$, the size of the vocabulary, where usually $D \ll V$. This mismatch is known to raise risks of limited expressivity in neural LMs, creating a so-called softmax bottleneck.
We show the softmax bottleneck is not only an expressivity bottleneck but also an optimization bottleneck. Backpropagating $V$-dimensional gradients through a rank-$D$ linear layer induces unavoidable compression, which alters the training feedback provided to the vast majority of the parameters. 
We present a theoretical analysis of this phenomenon and measure empirically that 95-99\% of the gradient norm is suppressed by the output layer, resulting in vastly suboptimal update directions. 
We conduct controlled pretraining experiments showing that the gradient bottleneck makes trivial patterns unlearnable, and drastically affects the training dynamics of LLMs. 
We argue that this inherent flaw contributes to training inefficiencies at scale independently of the model architecture, and raises the need for new LM head designs. 
%


\end{abstract}

\section{Introduction}\label{sec:intro}

There is significant research effort focusing on developing new architectures for the hidden layers of language models (LMs), aiming to improve efficiency and performance at training and inference times~\citep[e.g.][]{gu2024mamba,ye2025differential,yang2025gated}.
Despite fundamental differences, all considered architectures share a common structure for the output layer: a single linear mapping, followed by a softmax. 
This component is often called the \emph{LM head}. 
In other words, most autoregressive LMs can be seen as (potentially massive) feature extractors for a multi-class classifier, where classes correspond to tokens.

The LM head design is a standard choice for multi-class classification. 
However, the language modeling  classification setup itself is fairly unusual: the number of extracted features, i.e. the hidden dimension $D$ of the model, is commonly orders of magnitude smaller than the number of classes, i.e. the token vocabulary size $V$. 
This mismatch has been shown to limit the expressivity of LMs, and may lead to representation degeneration and performance saturation for small models~\citep{softmax_bottleneck,pmlr-v97-ganea19a,godey2024why}.
Yet, an important aspect of this mismatched mechanism has so far been ignored: 
\emph{the loss gradients in the high-dimensional logit space are backpropagated to a lower-dimension space before the backward pass reaches the rest of the layers}.
In this work, we frame the softmax bottleneck not as an expressivity issue but instead as a significant factor in training dynamics, where stronger bottlenecks hurt the data efficiency of LMs regardless of the underlying architectural choices.
\begin{figure}[t!]
    \centering
    \begin{subfigure}[t]{0.49\linewidth}
        \centering
        \includegraphics[width=\linewidth]{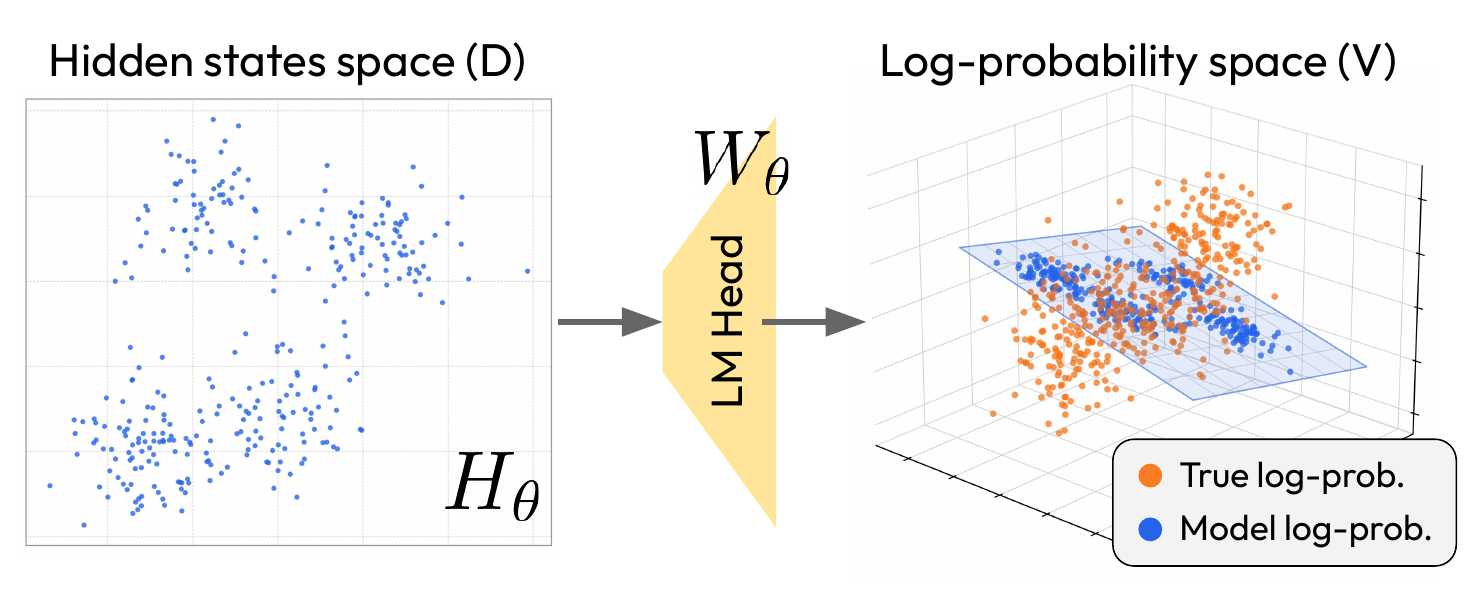}
        \caption{Expressivity view}
        \label{fig:schema_express}
    \end{subfigure}
    \hfill
    \begin{subfigure}[t]{0.49\linewidth}
        \centering
        \includegraphics[width=\linewidth]{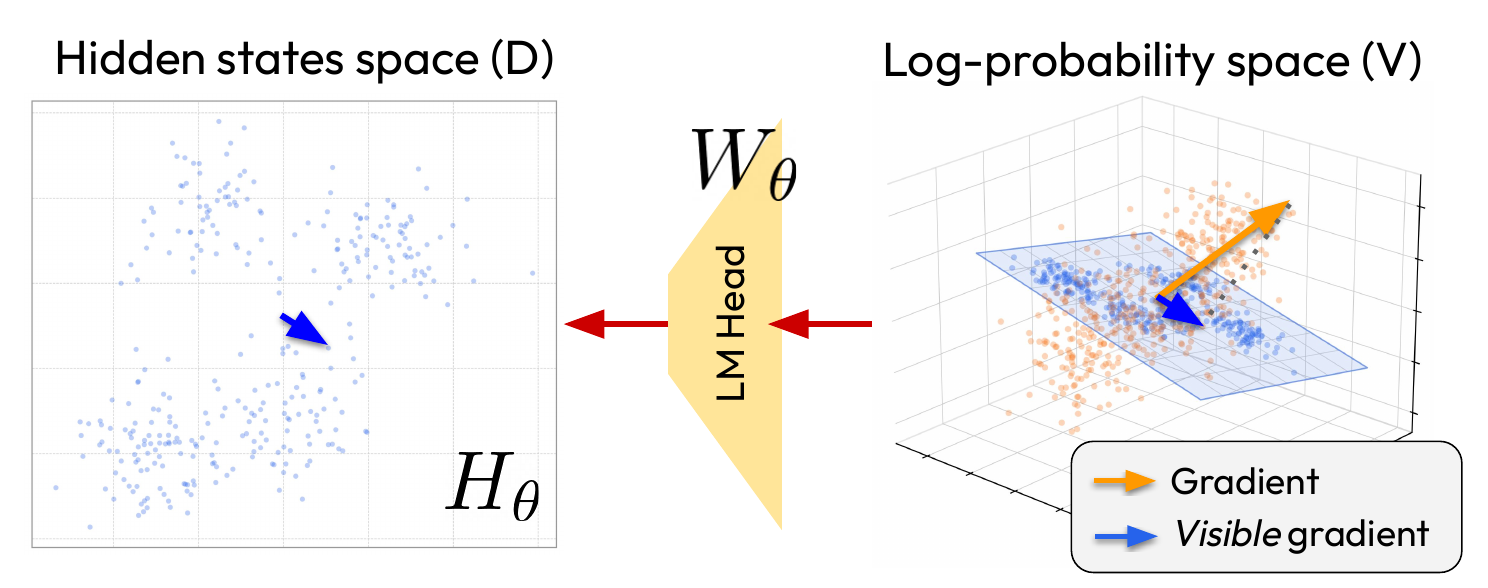}
        \caption{Optimization view}
        \label{fig:schema_optim}
    \end{subfigure}
    \caption{Two views of the softmax bottleneck.
\textbf{(a) Expressivity view:} The LM head $W_\theta$ projects hidden states $H_\theta$
into a rank-$D$ subspace of the $V$-dimensional log-probability space, constraining
where model log-probabilities ({\color{blue}$\bullet$}) can lie relative to the
true log-probabilities ({\color{orange}$\bullet$});
\textbf{(b) Optimization view:} During backpropagation, the full logit gradient
(\textcolor{orange}{$\rightarrow$}) lives in the $V$-dimensional space, but only
a projected component (\textcolor{blue}{$\rightarrow$}) is visible when
passing through $W_\theta$ back to the hidden state space.}
\label{fig:schema}
\end{figure}

We show both empirically and theoretically that the softmax bottleneck induces lossy compression during backpropagation, destroying $>$95\% of the gradient norm. Low-rank LM heads hamper optimization dynamics, making some trivial patterns unlearnable (\Cref{ssec:exp:spamlang}), and reducing LLM training efficiency by up to $\times 16$ for the same backbone (\Cref{ssec:exp:llm_training}).
Our work sheds light on a so far overlooked weakness in current LLM design, outlining important directions for future work.\footnote{Code, data and checkpoints will be released shortly.} 

Our main contributions are:

\begin{itemize}
    \item A theoretical analysis of the gradient bottleneck at the LM head, analyzing the structure of the logit gradient under mild assumptions, showing that part of the gradient is inevitably suppressed during backpropagation;
    \item A set of 2B parameter LMs with varying controlled bottlenecks, showing that training convergence speed is increasingly affected as the bottleneck gets stronger, with noticeable gaps up to our maximum tested hidden dimension (4{,}096);
    \item An example of a trivial synthetic language where expressivity is not an issue, but for which the gradient bottleneck makes it increasingly harder for the model to learn correctly;
    \item Empirical demonstration that the gradient bottleneck destroys 95-99\% of the gradient norm, and that it transfers energy from the most important part of the backpropagation signal to the tail of the coefficients in the form of random noise.
\end{itemize}



\section{Theoretical Overview}\label{sec:theory_ow}

The softmax bottleneck has historically been framed as an \emph{expressivity} problem:
when the hidden dimension $D$ is much smaller than the vocabulary size $V$, the LM head
imposes a low-rank constraint on the output log-probabilities, preventing the model from
representing arbitrary next-token distributions~\citep{softmax_bottleneck}.
We show that this framing is incomplete.
As illustrated in~\Cref{fig:schema}, the same bottleneck severely distorts the gradients that propagate back through the LM
head during training, independently of any expressivity concern.
The full formal development and all proofs are in \Cref{sec:theory,app:proofs}; here we provide
the key intuitions and state our main results.

\paragraph{Problem Setup}
Let \mbox{$C$} be the number of distinct contexts in the training data, where contexts are defined as sequences of past tokens \mbox{$\mathbf{x}_{<t}$} taken in a vocabulary of size \mbox{$V$}. The model backbone (typically a Transformer) with parameters \mbox{$\theta$} produces a hidden state \mbox{$H_{\theta, i} \in \mathbb{R}^D$} for each context \mbox{$i \in [1,C]$}; we stack these into the \emph{hidden state matrix} \mbox{$H_\theta \in \mathbb{R}^{C \times D}$}. The \emph{LM head} \mbox{$W_\theta \in \mathbb{R}^{V \times D}$} projects each representation to logits, giving the \emph{logit matrix} \mbox{$L_\theta = H_\theta W^\top_\theta \in \mathbb{R}^{C \times V}$}. The softmax function \mbox{$\sigma$} computes next-token probabilities: \mbox{$P_\theta = \sigma(L_\theta) \in \Delta^{C \times V}$}, where \mbox{$\Delta^{m \times n}$} is the set of row-stochastic matrices in \mbox{$\mathbb{R}^{m \times n}$}. The training loss is the cross-entropy between \mbox{$P_\theta$} and the empirical next-token distribution obtained by counting next-tokens over all contexts \mbox{$\tilde{N} \in \Delta^{C \times V}$}. The key structural fact is that the logit matrix \mbox{$L_\theta$} has rank at most \mbox{$D$}, since it is a product of a \mbox{$C \times D$} and a \mbox{$D \times V$} matrix.

\subsection{Expressivity: the Classical Softmax Bottleneck View}\label{sec:theory_ow:expressivity}

The classical softmax bottleneck~\citep{softmax_bottleneck} occurs when $D \ll V$ and the model cannot represent arbitrary next-token distributions:
$\operatorname{rank}(L_\theta) \leq D$ forces the output log-probabilities into a low-dimensional
subspace of $\mathbb{R}^V$. As a result, if the log-probabilities associated with the true next-token distribution span a space of rank $>D$ in $\mathbb{R}^V$, a neural LM \emph{cannot} perfectly match the true distribution~\citep{softmax_bottleneck}.

However, this constraint is weaker than it might appear. 
When the hidden dimension $D \ge 2$, a neural LM as described above has sufficient expressivity to assign the correct top-1 next-token probability
to arbitrary precision, regardless of vocabulary size $V$. 
We formally show this in \Cref{prop:top1} and \Cref{proof:top1}.
The low-rank constraint therefore does not prevent the model from predicting
\emph{which} token comes next, nor from assigning it the right probability. 
Recent work extends this result to show that even more complex truncated probability distributions can be expressed even when $D \ll V$~\citep{basri2026the}. 

Nevertheless, the fact that a LM has the theoretical expressivity to predict a token with the correct empirical probability does not mean that it can easily be optimized to do so, leading to studying the optimization viewpoint.

\subsection{Optimization: Logit Updates Are Rank-Constrained}
\label{sec:theory_ow:gradientdescent}

We analyze training through the lens of how the logits $L_\theta$ evolve under gradient descent.
The first-order optimal update to the logits is the logit gradient
\mbox{$\nabla_L \mathcal{L} = \operatorname{diag}(f)(P_\theta - \tilde{N})$} where $f$ is the context frequency vector in $\mathbb{R}^C$, $P_\theta$ is the model next-token probability matrix, and $\tilde{N}$ is the empirical next-token probability matrix.
Training is efficient in the first order when the updates for the LM head $W_\theta$ and the hidden states $H_\theta$ move the logit matrix $L_\theta$ along its gradient $-\nabla_L \mathcal{L}$.

The actual update to the logit matrix $L_\theta$ induced by gradient descent on the projection parameters $W_\theta$, however, is constrained.
Because any update $\Delta W_\theta$ to the LM head produces a logit change $H_\theta \Delta {W_\theta}^\top$
of rank at most $D$, and similarly for the contribution from updating the  state $H_\theta$, the logit update has rank at most $2D$. As the logit gradient lives in $\mathbb{R}^{C \times V}$, it may  have a rank up to $V$, which is usually much larger than $2D$. We show that the logit gradient is structurally high-rank, implying that the logit update cannot align with its direction.

\begin{resultbox}[(informal; \Cref{eq:delta,prop:upd_res})]{}
    \label{result:fullgd_error}
     Under mild conditions about the data distribution, the logit gradient $\nabla_L\mathcal{L}$ has full rank $V$. Hence, the update direction is provably suboptimal, with a residual lower--bounded by the tail singular values of $(P_\theta - \tilde{N})$ beyond rank $2D$.
\end{resultbox}

   

\noindent Our main theoretical contribution is in the proof that the logit gradient is structurally high-rank. 
\Cref{proof:rankPN} provides the complete proof.
Informally, we extract from the logit gradient a submatrix analogous to a fully-connected graph laplacian up to a permutation of rows and a potential deletion of columns that correspond to specific tokens. When no columns have to be deleted to obtain the appropriate submatrix, the logit gradient inherits the property of the laplacian and is near full-rank ($V-1$). When columns have to be deleted, we show that this submatrix is strictly diagonally dominant and is thus also full rank by Gershgorin's circle theorem, where the rank depends on the number of remaining columns.

\subsection{The Bottleneck Persists Under SGD and Alternative Heads}\label{sec:theory_ow:sgd}


\paragraph{Stochastic Gradient Descent}
One might hope that mini-batches include fewer contexts, reducing the effective rank of \mbox{$(P_\theta - \tilde{N})$} taken at batch-level and relaxing the bottleneck. However, we show that this is not the case.

\begin{resultbox}[(informal; \Cref{prop:sgd,cor:sgd})]{result:sgd_error}
 Near convergence, the prediction error within each mini-batch has rank close to
$V - 1$, as in the full-batch setting.
\end{resultbox}

\noindent The proof (\Cref{proof:sgd}) emerges from considering contexts where the empirical next-token probability vector is dense when measured over the full dataset but sparse when measured over the current batch, which is very likely to happen in natural language. In such cases, the logit gradient itself is dense, and under sufficient density-related conditions, the rank of the logit gradient is near full rank.

\paragraph{Alternative LM Head Designs}
Prior work proposed mixture-based or nonlinear output layers to boost the expressible
rank of the output log-probabilities~\citep{softmax_bottleneck,mixtape,pmlr-v97-ganea19a}.
These address the expressivity bottleneck but not the optimization bottleneck.
For any mapping from the hidden states to the logits $g_\theta(H_\theta) = L_\theta$, the gradient reaching the hidden states
is $\nabla_H \mathcal{L} = \nabla_L \mathcal{L} \cdot J_g(H_\theta)$, where the
Jacobian $J_g$ has rank at most $D$ regardless of the choice for $g$ (\Cref{eq:jacobian}). It is possible that some choices for $g$ lead to better conditioning, but the prior work on softmax bottleneck does not explicitly target such choices. Designing optimization-compliant $g$ is an important direction for future work.

\subsection{Empirical Validation \& Discussion}\label{sec:theory_ow:empirical}
The residual bound in \Cref{result:fullgd_error} is stated in terms of singular values, which raises a legitimate
concern: even if $\operatorname{rank}(P_\theta - \tilde{N}) > 2D$, the tail singular
values could be negligible, making the bound vacuous. Our theory therefore establishes
the \emph{inevitability} of a compression in the backpropagation through the LM head,
but not its severity. We close this gap empirically through three complementary
observations.

First, we verify that logit gradients are near full-rank in practice: the empirical rank of
$\nabla_L \mathcal{L}$ for Pythia models~\citep{biderman2023pythia} on the
Pile~\citep{gao2020pile} grows rapidly with batch size and approaches $V$
(\Cref{fig:pythia_grad_rank}), validating empirically the results on the rank of the logit gradient in  \Cref{result:fullgd_error,result:sgd_error}.

While the bound obtained in \Cref{result:fullgd_error} is non-vacuous, the singular value spectrum of
$\nabla_L \mathcal{L}$ is not flat. \Cref{fig:svd_components} shows that
the number of SVD components needed to almost reconstruct the logit gradient up to 99.9\% is substantial (up to 30{,}000), but a relatively small number of dimensions
already captures a substantial share (50--70\%) of the signal. In other words, a rank-$D$
projection that actively targets the dominant directions of $\nabla_L \mathcal{L}$
could theoretically recover a meaningful portion of the gradient.

However, as we discuss in \Cref{ssec:exp:grad_comp}, $W_\theta$ does not optimize for this compression empirically. Projecting
$\nabla_L \mathcal{L}$ onto the null space of ${W_\theta}^\top$ destroys $95$--$99\%$ of the gradient norm across GPT-2, Pythia, Llama~3, OLMo~2, and Qwen~3, with a cosine similarity between the surviving signal and the original of only $0.1$--$0.2$
(\Cref{fig:lost_gradient,fig:cossim_grad}). The backpropagation through ${W_\theta}^\top$ is thus destructive: rather than preserving the dominant gradient directions, it discards nearly all of the signal, redirecting energy to the tail of the
coefficients in the form of noise.



\section{Experiments}\label{sec:experiments}

\subsection{Consequences for LLM Training}\label{ssec:exp:llm_training}

We measure the impact of the gradient bottleneck effect for actual pretraining setups. A difficulty that arises in such experiments is the inherent dependency of model size on the hidden dimension. As a result, studying models with different $D$ values either implies comparing models of different sizes, or adjusting the depth of the models to account for the increased width. However, prior work~\citep{petty-etal-2024-impact} reports that deeper models show stronger language modeling and downstream performance, which may be a confounding factor when adjusting depth and width in such a way.

To disentangle the impact of architectural choices and of the gradient bottleneck effect on performance, we train LMs that share the same Transformer stack, based on the Llama3 architecture, with a hidden dimension of $d_m=4096$ and 6 hidden layers, representing a total of 2B parameters (or 1.8B without the output embeddings). However, the language modeling head of each model mimics a hidden dimension $D \leq d_m$ by adopting a low-rank structure $W_\theta = A_\theta B_\theta$, where $A_\theta \in \mathbb{R}^{V \times D}$ and $B_\theta \in \mathbb{R}^{D \times d_m}$. As a result, the Transformer backbone is equally expressive across all compared models, but $D$ provides control over the strength of the gradient bottleneck.

\begin{figure}[t!]
    \centering
    \begin{subfigure}[t]{0.49\linewidth}
        \centering
        \includegraphics[width=\linewidth]{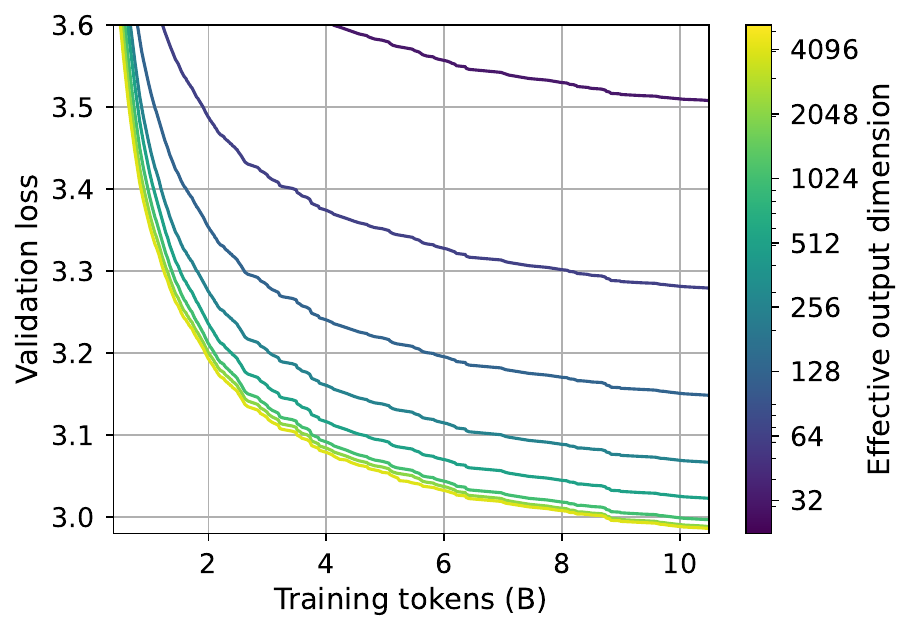}
        \caption{Validation loss.}
        \label{fig:training_dyn}
    \end{subfigure}
    \hfill
    \begin{subfigure}[t]{0.49\linewidth}
        \centering
        \includegraphics[width=\linewidth]{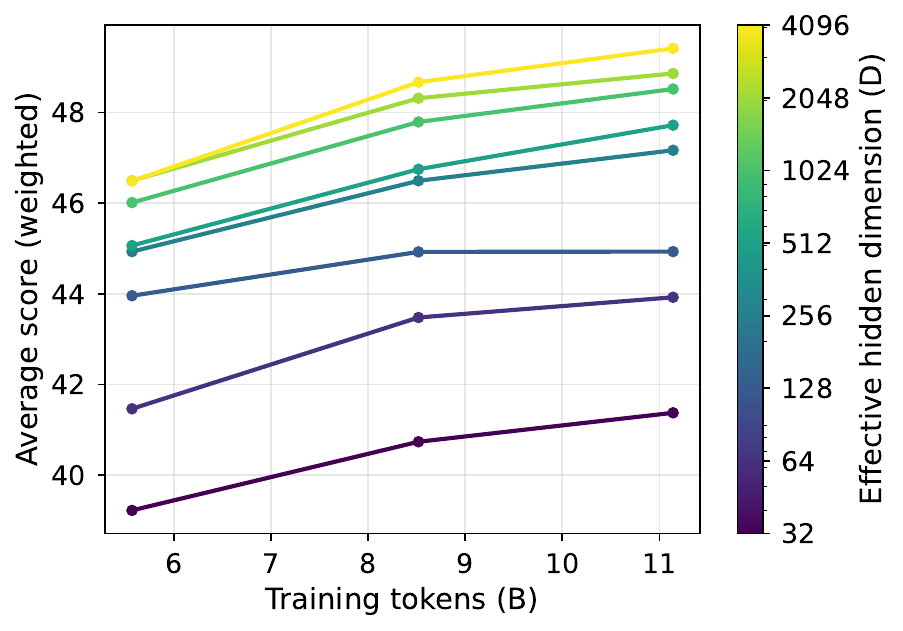}
        \caption{Downstream zero-shot scores.}
        \label{fig:benchmark_perf}
    \end{subfigure}
    \caption{Impact of the gradient bottleneck on 2B model training. We control the rank $D$ of the output linear layer to isolate the bottleneck effect without changing the Transformer backbone. Both validation loss and downstream performance improve consistently with $D$, with a $\times$16 convergence speedup between $D{=}32$ and $D{=}4096$.}
    \label{fig:llm_training}
\end{figure}

We train 8 such 2B parameter models on $\sim$11B tokens of Fineweb-Edu data~\citep{fineweb}, with $D=2^i$ for $i\in[5, 12]$. To avoid exacerbating the effects of the gradient bottleneck by choosing a large vocabulary, we use the relatively compact SmolLM2 tokenizer~\citep{allal2025smollm2smolgoesbig} that counts $V=\num{49152}$ tokens. We use the WSD learning rate schedule~\citep{hu2024minicpm} to perform cooldown at different points in training (5B, 8.5B, 11B). All details about the pretraining setup are in \Cref{app:exp_setups}.

\Cref{fig:training_dyn} shows the validation loss evolution along training. We perform zero-shot downstream evaluation on our models after LR cooldown in \Cref{fig:benchmark_perf}. It clearly appears that both the language modeling and downstream performance converge faster for larger values of $D$. The validation loss of all models is separated by consistent and non-negligible gaps along training. The $D=\num{4096}$ reaches the final loss level of the $D=32$ variant within 700M tokens, displaying a $\times16$ speedup in convergence. The final average score gap for $D=\num{2048}$ and $D=\num{4096}$ is +0.55, showing a noticeable performance gap even for large hidden dimension.  We report detailed scores in \Cref{app:bench_scores}.

A potential confounding factor for this experiment is the slight difference in total parameter count across models (ranging from 1.8B to 2.0B parameters). However, it is not clear whether counting embedding parameters is relevant for downstream performance extrapolation~\citep{kaplan_scaling,chinchilla_scaling}, and the reported gaps are much more significant than such size discrepancies would predict.

\subsection{Disentangling Expressivity and Optimization Issues}
\label{ssec:exp:spamlang}


\begin{figure}[t!]
    \centering
    \begin{subfigure}[t]{0.51\linewidth}
        \centering
        \includegraphics[width=\linewidth]{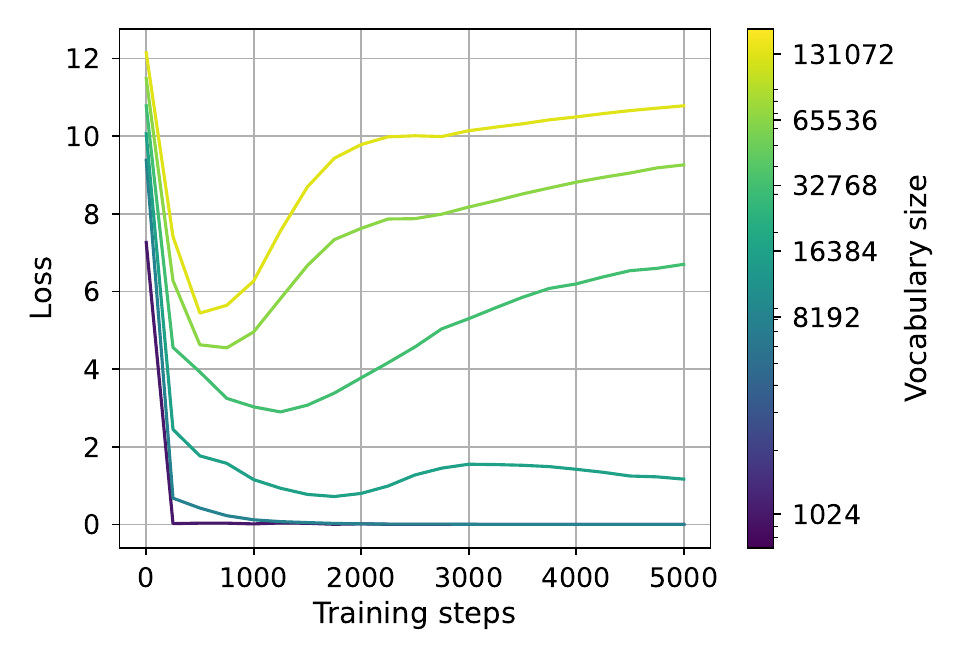}
        \caption{Training curves (LR: 5e-4).}
        \label{fig:spamlang_training_curves}
    \end{subfigure}
    \hfill
    \begin{subfigure}[t]{0.47\linewidth}
        \centering
        \includegraphics[width=\linewidth]{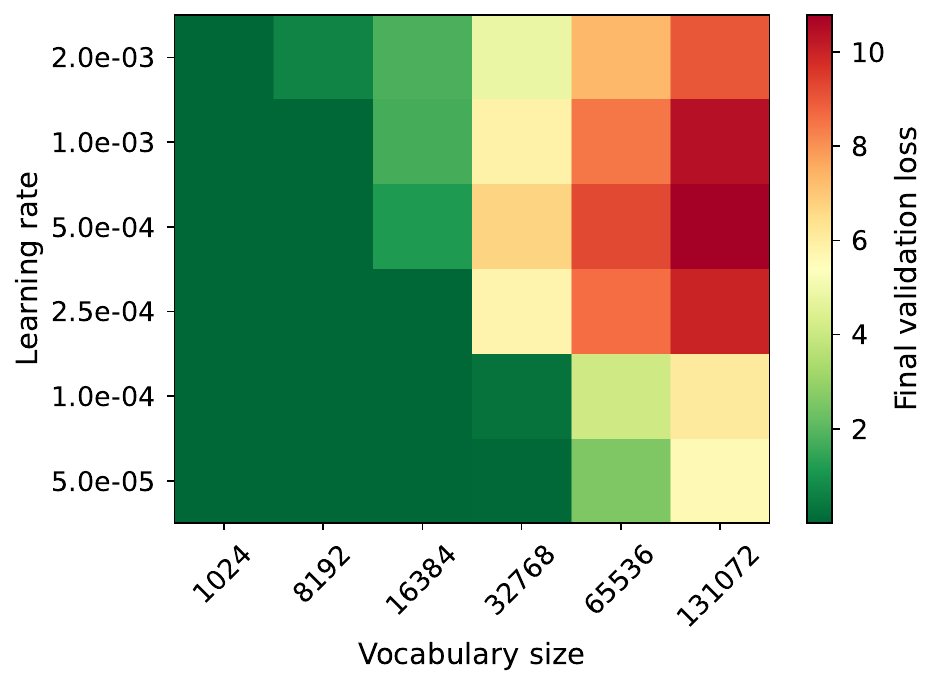}
        \caption{Final validation loss across vocabulary sizes and learning rates ($D{=}576$).}
        \label{fig:spamlang_final}
    \end{subfigure}
    \caption{\texttt{SpamLang} experiments. Despite Transformers being provably expressive enough to learn this synthetic language for any vocabulary size (\Cref{prop:top1}), the gradient bottleneck makes learning increasingly difficult as $V$ grows.}
    \label{fig:spamlang}
\end{figure}

Framing the softmax bottleneck problem solely as an expressivity limitation does not account for all the phenomena observed when $D \ll V$. 
We illustrate this by considering a trivial synthetic language (\texttt{SpamLang}) made of $V$ symbols, where each sequence consists of one uniformly drawn symbol repeated throughout the whole sequence: 
\[(w_1, w_1, ..., w_1), (w_2, w_2, ..., w_2), ...\]

By virtue of \Cref{prop:top1}, and given that Transformers can perfectly learn the identity function~\citep{Yun2020Are}, for any $V$ and $D \geq 2$, we can manually pick the weights of a Transformer-based model to reach a loss arbitrarily close to $0$. Hence, Transformers have sufficient expressivity for this problem. 
We train a Transformer-based LM with 106M non-embedding parameters and $D = \num{576}$ on 41M tokens of \texttt{SpamLang} for vocabulary sizes $V$ ranging from 1{,}024 to 131{,}072. Critically, even for the largest vocabulary size, each distinct token should occur roughly 300 times in expectation.

\begin{wraptable}{r}{0.45\linewidth}
\vspace{-1em}
\centering
\begin{tabular}{r l}
\toprule
\textbf{V} & \textbf{Generated sequence} \\
\midrule
16384 & \small{\texttt{A}}\colorchar{A}\colorchar{A}\colorchar{A}\colorchar{A}\colorchar{A}\colorchar{A}\colorchar{A}\colorchar{A}\colorchar{A}\colorchar{A}\colorchar{A}\colorchar{A}\colorchar{A}\colorchar{A}\colorchar{A}\colorchar{A}\colorchar{A}\colorchar{A}\colorchar{A} \\
\midrule
32768 & \small{\texttt{A}}\colorchar{A}\colorchar{Ġ}\colorchar{p}\colorchar{A}\colorchar{Ġ}\colorchar{A}\colorchar{A}\colorchar{Ġ}\colorchar{A}\colorchar{Ġ}\colorchar{i}\colorchar{Ġ}\colorchar{A}\colorchar{A}\colorchar{'}\colorchar{Ġ}\colorchar{Ġ}\colorchar{A}\colorchar{Ġ} \\
\midrule
65536 & \small{\texttt{A}}\colorchar{-}\colorchar{|}\colorchar{l}\colorchar{u}\colorchar{-}\colorchar{-}\colorchar{|}\colorchar{B}\colorchar{-}\colorchar{|}\colorchar{A}\colorchar{N}\colorchar{Ġ}\colorchar{R}\colorchar{Ġ}\colorchar{b}\colorchar{.}\colorchar{Ġ}\colorchar{|} \\
\midrule
131072 & \small{\texttt{A}}\colorchar{ç}\colorchar{.}\colorchar{|}\colorchar{ç}\colorchar{Ġ}\colorchar{ç}\colorchar{ç}\colorchar{ç}\colorchar{ç}\colorchar{y}\colorchar{ç}\colorchar{ç}\colorchar{æ}\colorchar{m}\colorchar{ç}\colorchar{è}\colorchar{×}\colorchar{ç}\colorchar{ç} \\
\bottomrule
\end{tabular}
\vspace{5pt}
\caption{Generated samples from our \texttt{SpamLang} models for different vocabulary sizes (LR: 2.5e-4). We use a single token (here \texttt{A}) as a seed and generate with temperature 0.9. We use the first character of the Qwen3-Base tokens to represent \texttt{SpamLang} tokens. We use colors to label \texttt{seed}, and \textcolor{green(html/cssgreen)}{\small{\texttt{correct}}}/\colorchar{incorrect} predictions.}
\label{tab:genrepeat}
\vspace{-0.5em}
\end{wraptable}

\Cref{fig:spamlang_training_curves} reports the loss on the training distribution across training with a learning rate of 5e-4. Convergence is increasingly difficult as the vocabulary size increases compared to our fixed hidden dimension. To mitigate the confounding effect of optimization hyperparameters, we explore various learning rates in \Cref{fig:spamlang_final}. Moreover, we report results with and without embedding weight tying in \Cref{app:exp_setups}, as reaching the solution described in the proof of \Cref{prop:top1} may be easier in the case of weight tying. The results show that although the models are robust to most learning rate choices for smaller vocabulary sizes, they tend to become more and more sensitive to high learning rates, up to a point where no learning rate choice (in the tested range) leads to proper convergence. In \Cref{tab:genrepeat}, we generate using our models for increasing vocabulary sizes to illustrate the failures of models when $V$ is large. We also notice that the repetition mechanism is only partially learned, as models sometimes repeat tokens for a few steps only.

This experiment shows that expressivity is not a sufficient scope to understand the softmax bottleneck issue. It adequately illustrates how increasing $V$ and thus the dimensionality of the logits gradients makes it difficult -- and in cases impossible -- to learn basic patterns as the supervision feedback gets more and more compressed.

\subsection{Analysis of the Gradient Compression}\label{ssec:exp:grad_comp}

A way to characterize empirical patterns of gradient compression in the common LM pretraining setup is studying the null space of the projection matrix $W_\theta^\top$ (or the more general projection of the hidden state $J_g(H_\theta)$), as:
\[
\nabla_H\mathcal{L} =  \nabla_L\mathcal{L} \cdot J_g(H_\theta) =  \nabla_L\mathcal{L} \cdot W_\theta\;\;.
\]
Hence, the projection of the logits gradient on $\ker (W_\theta^\top)$ is exactly the part of the logits gradient that is destroyed by the gradient bottleneck.
We extract a basis for $\ker (W_\theta^\top)$ for different pretrained models by performing a $QR$ decomposition of $W_\theta$ and by selecting the trailing $V-D$ columns of $Q$. We can measure the fraction of the logits gradient norm that is lost in the backpropagation process:
\begin{equation}
\label{eq:lostgrad}
\frac{||p_{\ker(W_\theta^\top)} (\nabla_L \mathcal{L})||_F}{||\nabla_L \mathcal{L}||_F}\;\;.
\end{equation}

\begin{figure}[t]
    \centering
    \begin{subfigure}[t]{0.49\linewidth}
        \centering
        \includegraphics[width=\linewidth]{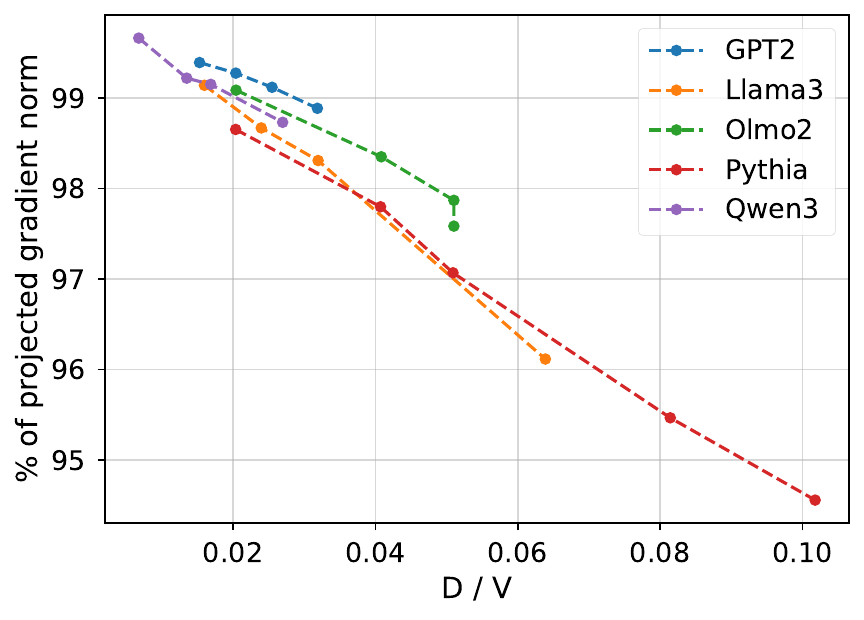}
        \caption{Fraction of projected gradient norm.}
        \label{fig:lost_gradient}
    \end{subfigure}
    \hfill
    \begin{subfigure}[t]{0.49\linewidth}
        \centering
        \includegraphics[width=\linewidth]{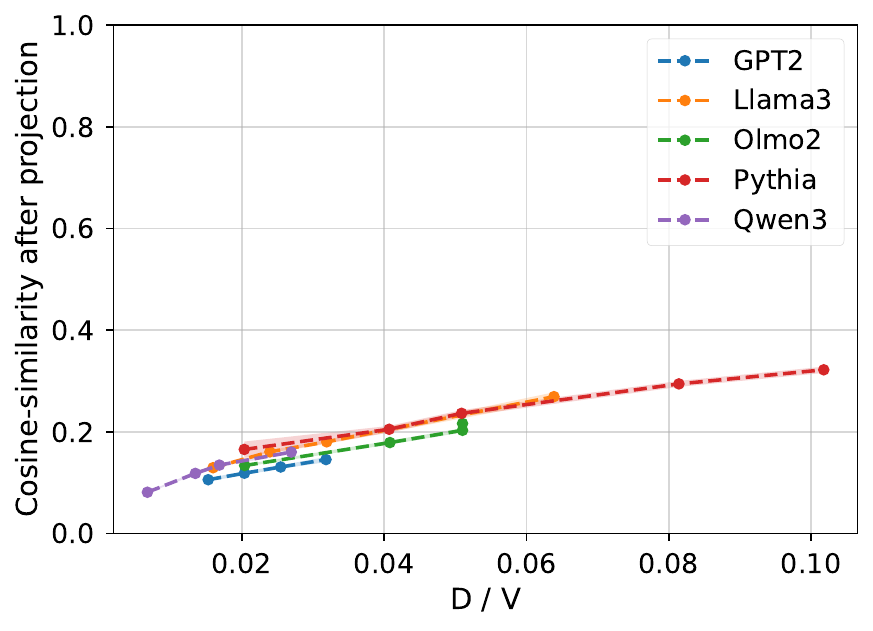}
        \caption{Cosine similarity after projection.}
        \label{fig:cossim_grad}
    \end{subfigure}
    \caption{Gradient compression as a function of $D/V$, for several model families. The logit gradient norm is almost entirely destroyed by backpropagation through $W_\theta$, and the surviving projected gradient is only weakly aligned with the original.}
    \label{fig:gradient_compression}
\end{figure}

\Cref{fig:lost_gradient} reports the fraction of the gradient norm that is destroyed as measured by \Cref{eq:lostgrad} for several model families: GPT2~\citep{radford2019language}, Pythia~\citep{biderman2023pythia}, Llama3~\citep{llama3modelcard}, OLMo2~\citep{olmo20252olmo2furious}, and Qwen3-Base~\citep{qwen3technicalreport}. We use 10{,}000 shuffled documents from the FineWeb dataset~\citep{fineweb} as a proxy for usual web-crawled pretraining datasets. We observe that the gradients are almost completely suppressed for all model families, with $\sim$95-99\% of their norm being projected to $\ker (W_\theta^\top)$. In \Cref{fig:cossim_grad}, we report the average cosine similarity between the logit gradient and its projection on the complement of $\ker(W_\theta^\top)$, which can be framed as the part of the gradient that is visible for all preceding parameters. We observe that most values are comprised between 0.1 and 0.3 showing that the projection is in general mildly aligned with the original gradient. All families follow a similar trend where both the fraction of lost logit gradient norm diminishes as $D$ increases. We perform a similar analysis across intermediate checkpoints of OLMo2-1B in \Cref{app:training_dynamics} to investigate the training dynamics of these metrics, showing that they remain nearly constant throughout training.

\begin{figure*}[t!]
    \centering
    \begin{subfigure}[t]{0.49\textwidth}
        \centering
        \includegraphics[width=\textwidth]{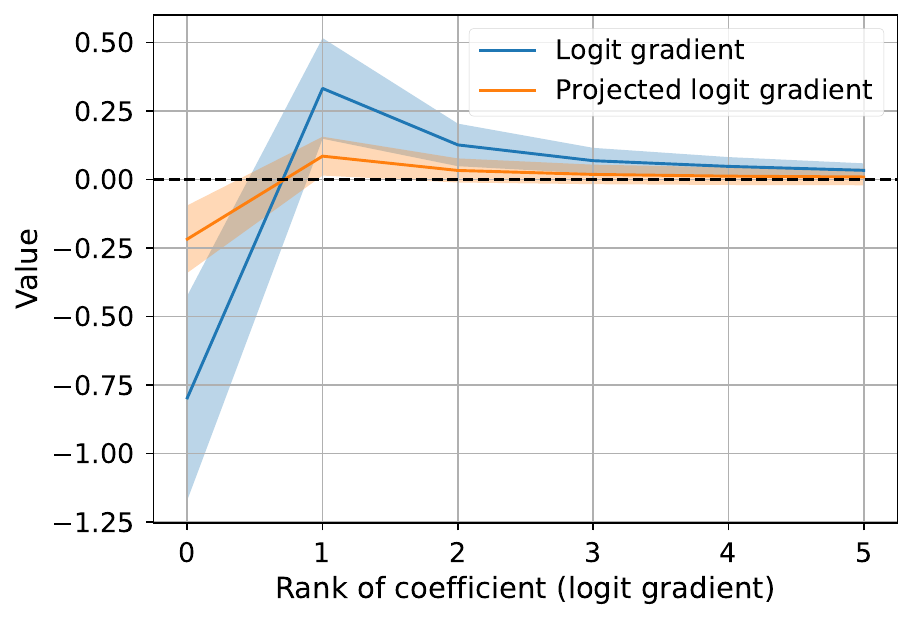}
    \end{subfigure}%
    ~ 
    \begin{subfigure}[t]{0.49\textwidth}
        \centering
        \includegraphics[width=\textwidth]{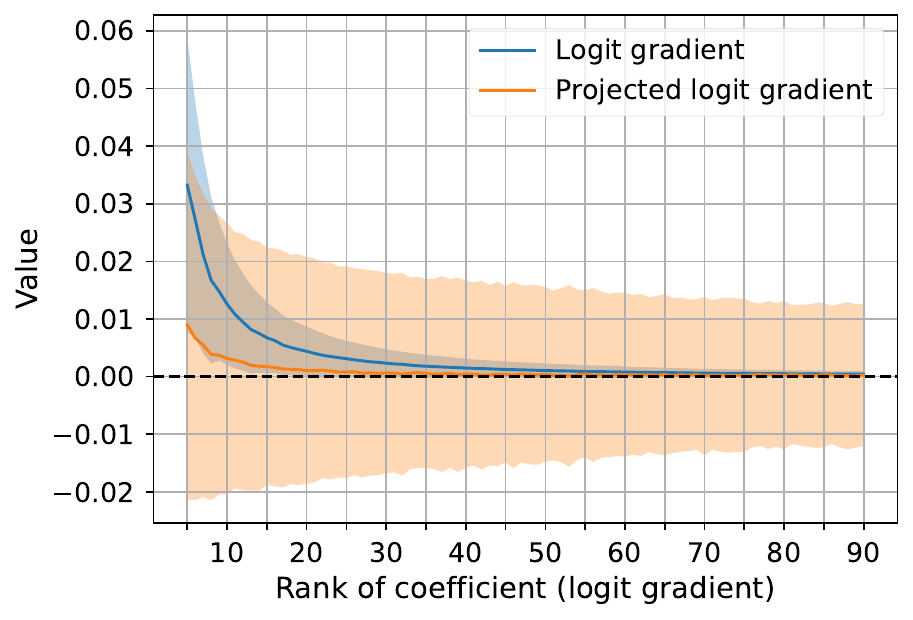}
    \end{subfigure}
    \caption{Average full and projected logits gradient coefficients, sorted by absolute values of the full gradient. The gradients are taken for Llama-3.1-70B on Fineweb documents. The plot on the left-hand side focuses on the main coefficients, while the plot on the right-hand side zooms in on smaller coefficients. The standard deviation is reported using filled areas around the curves.}
    \label{fig:logit_grad_shape}
\end{figure*}

\Cref{fig:logit_grad_shape} directly compares the distributions of the coefficients of the full and projected logits gradients for Llama-3.1-70B. The negative coefficient corresponding to the observed token maintains its sign during the projection. The pattern for the first 5--10 coefficients is strongly diminished but still correlated with the full gradient. However, the fraction of the norm dedicated to the tail of the coefficients -- corresponding to tokens that were predicted with low probability -- is more important for the projected gradient. It takes the form of a higher variance, which can be viewed as a form of noise on the training feedback. We perform the same analysis for Llama-3.1-8B and OLMo2-32B in \Cref{app:grad_comp} and reach similar conclusions. In \Cref{app:gradcomp_ex}, we provide a fine-grained token-level analysis of this distortion on a simple example, showing strong and noticeable misalignment and exposing interesting semantic and lexical patterns.

\subsection{Efficiency of the Update Direction}
\label{ssec:exp:update_dir}

The training feedback suffers from a destructive compression at the output layer level, which implies a loss of information during the backpropagation process. As a result, the hidden states $H_\theta$ and the subsequent model weights are updated based on partial views of the training batch. This raises the question of the efficiency of optimization methods in that setup: to what extent does descending in that compressed direction yield a decrease of the loss on the current training batch?

We answer this question by comparing the loss dynamics for two update directions for the logits: the first-order optimal logits gradient $\nabla_L \mathcal{L}$, and the direction in logit space resulting from updating the hidden states along their gradient, i.e. $\nabla_H \mathcal{L} \cdot W_\theta^\top$. In \Cref{fig:direction_eff}, we extract both directions for several trained models on batches of 1{,}024 tokens of Fineweb data and compute the loss variation resulting from updating the logits by a fraction of their norm in each direction.

\begin{figure*}[t!]
    \centering
    \begin{subfigure}[t]{0.33\textwidth}
        \centering
        \includegraphics[width=\textwidth]{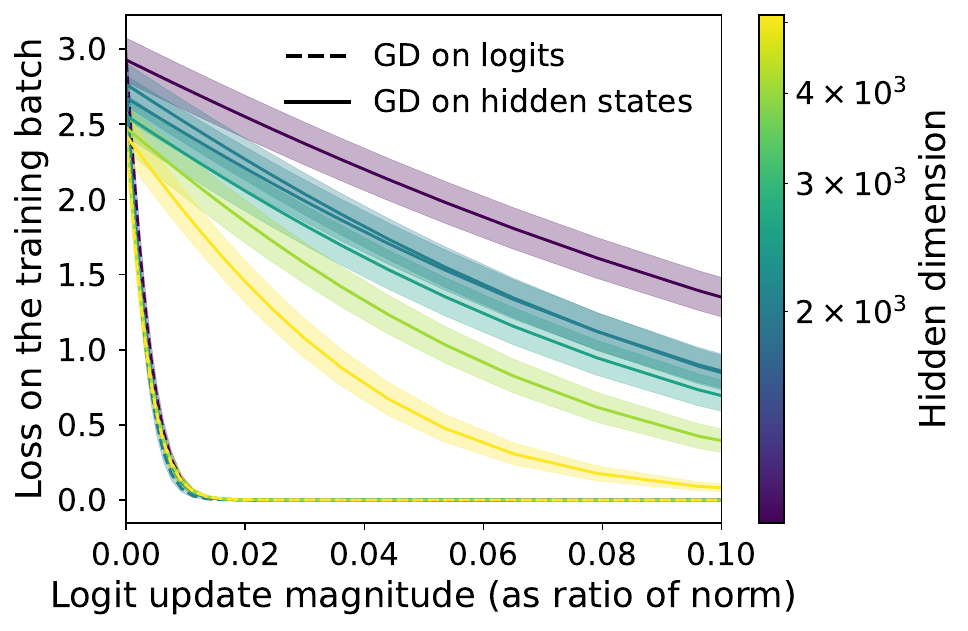}
        \caption{Pythia}
    \end{subfigure}%
    ~ 
    \begin{subfigure}[t]{0.33\textwidth}
        \centering
        \includegraphics[width=\textwidth]{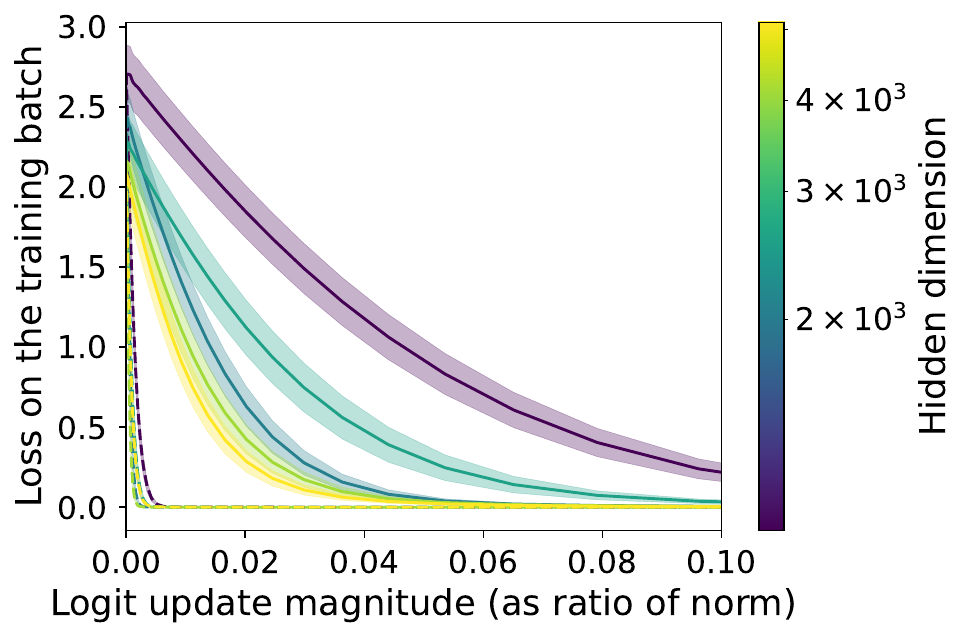}
        \caption{Qwen3-Base}
    \end{subfigure}%
    ~ 
    \begin{subfigure}[t]{0.33\textwidth}
        \centering
        \includegraphics[width=\textwidth]{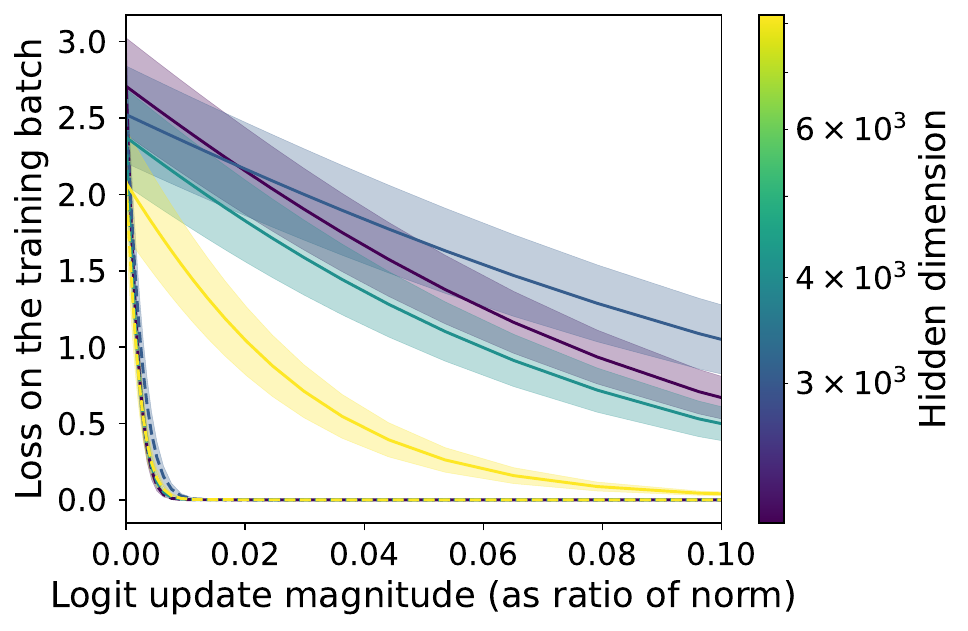}
        \caption{Llama3}
    \end{subfigure}%
    \caption{Effect of first-order updates on logits and hidden states for the language modeling loss for different families of models.}
    \label{fig:direction_eff}
\end{figure*}

For each family, \Cref{fig:direction_eff} clearly shows that updating hidden states along $\nabla_H\mathcal{L}$ is orders of magnitude less efficient than directly updating logits along $\nabla_L\mathcal{L}$, given the same budget for the logit update norm. Moreover, updating $H$ along $\nabla_H\mathcal{L}$ is more efficient as the hidden dimension $D$ increases, proving that the gradient compression becomes less harmful for update efficiency. This phenomenon is independent of model size, as the loss curves for the 1B and 1.4B Pythia models that share the same hidden dimension hyperparameter collapse into each other.
This shows that the gradient compression at the output layer of LLMs is destructive and redirects supervision feedback energy from top components to the tail of the gradient coefficients in the form of noise. This compression harms optimization because the resulting compressed gradients turn out to provide much less efficient update directions for the training loss.
\section{Related Works \& Discussion}
\label{sec:rw}
\paragraph{Softmax Bottleneck and Alternatives} \citet{softmax_bottleneck} are the first to show that softmax-based LMs face an expressivity limitation when the hidden dimension is smaller than the rank of the target log-probability matrix. \citet{chang-mccallum-2022-softmax} extended this analysis, showing that a single hidden state cannot simultaneously be close to all possible next-word embeddings, making models unable to represent multi-mode word distributions. \citet{grivas-etal-2022-low} show that some permutations of next-token probability ordering cannot be reached when $D < V$, but observe rare occurrences of such issues in LMs. To address these expressivity limitations, several architectural alternatives have been proposed~\citep{softmax_bottleneck,sigsoftmax,pmlr-v97-ganea19a}, which all consist in breaking the softmax function into several components to increase the logit rank. While these works focus on the expressivity consequences of the softmax bottleneck, our work reveals a complementary optimization perspective, harming training efficiency and potentially impeding convergence towards expressive weight configurations.

\vspace{-0.7em}
\paragraph{Representation Degeneration and Gradient Flow} \citet{gao2018representation} identified the representation degeneration problem in neural language generation, observing that output embeddings tend to concentrate in a narrow high-dimensional cone when trained on large datasets. \citet{godey2024why} link this issue to the softmax bottleneck, showing that the emergence of such degeneration can coincide with a performance drop for small LMs. They hypothesize that the softmax bottleneck can lead to optimization issues through unstable dynamics when the LM head reaches a form of spectral saturation. Our findings could shed light on these optimization dynamics. Moreover, even when the representations do not degenerate, the gradient bottleneck effect is still limiting the amount of information that can be backpropagated to the model.

\vspace{-0.7em}
\paragraph{Logit Gradient \& Equilibrium Analysis}
\citet{mircea2024gradient} study the loss saturation of Pythia models through an analysis of the logits gradients and expose a form of tug-of-war effect between positive and negative components that harms learning. Although we do not decompose the logits gradient similarly, we hypothesize that some connections should be made by considering the conservation of sign patterns under the low-rank gradient bottleneck. \citet{finlayson2024closing} expose how some patterns of next-token probability emerge through necessary token entanglement in the backpropagation process. While their framing of the softmax bottleneck issue is connected to ours, they only consider implications for expressivity at convergence, and they do not analyze the logits gradients dimensionality in depth.

\vspace{-0.7em}
\paragraph{Toward Mitigations} We conducted several preliminary experiments to mitigate the gradient bottleneck without redesigning the LM head, including regularizing $W_\theta$ toward orthogonality, adding an auxiliary loss penalizing misalignment between the logit update and the full logit gradient, and using feedback alignment~\citep{lillicrap2016random}. All approaches led to slower convergence. This underscores that the gradient bottleneck is not easily patched and warrants dedicated architectural innovation rather than post-hoc corrections.

\vspace{-0.7em}
\paragraph{Discussion \& Implications} Our findings show that neural LMs suffer from severe gradient compression at their output layer, which represents an inherent limitation for the training efficiency of LLMs. We hypothesize that taking hidden dimensions into account when computing scaling laws~\citep{chinchilla_scaling} could help refine their extrapolation quality. A more optimistic view of our work could also indicate that LMs have stronger potential than currently believed, and that convergence speed and/or performance gains could be obtained by better channeling the supervision signal through a better-suited logits prediction module. Finally, while our work focuses on characterizing the problem, it opens promising directions for innovations that better preserve gradient flow, whether through pre-conditioning, optimization techniques, or well-suited softmax alternatives. These avenues may offer substantial gains in training efficiency.

\section{Conclusion}\label{sec:conclusion}

We have demonstrated that the softmax bottleneck in neural language models is not merely an expressivity limitation but a fundamental optimization bottleneck. Our theory-grounded empirical analysis shows that 95--99\% of the supervision signal is lost during backpropagation through the output layer, in a transfer from informative components to the tail as random noise.
Through controlled experiments, we showed that this gradient compression can make even trivial patterns difficult to learn as vocabulary size grows, and significantly slows convergence in realistic 2B parameter pretraining runs. These findings suggest that current LMs train less efficiently than they could.
We hope this work inspires renewed attention to this overlooked but critical component of language model architecture.




\bibliography{biblio}

@article{lillicrap2016random,
  author    = {Lillicrap, Timothy P. and Cownden, Daniel and Tweed, Douglas B. and Akerman, Colin J.},
  title     = {Random synaptic feedback weights support error backpropagation for deep learning},
  journal   = {Nature Communications},
  volume    = {7},
  number    = {1},
  pages     = {13276},
  year      = {2016},
  doi       = {10.1038/ncomms13276},
  url       = {https://doi.org/10.1038/ncomms13276}
}

@inproceedings{
basri2026the,
title={The Softmax Bottleneck Does Not Limit the Probabilities of the Most Likely Tokens},
author={Ronen Basri and David Jacobs},
booktitle={The Fourteenth International Conference on Learning Representations},
year={2026},
url={https://openreview.net/forum?id=DgJqQk6y19}
}

@inproceedings{
yang2025gated,
title={Gated Delta Networks: Improving Mamba2 with Delta Rule},
author={Songlin Yang and Jan Kautz and Ali Hatamizadeh},
booktitle={The Thirteenth International Conference on Learning Representations},
year={2025},
url={https://openreview.net/forum?id=r8H7xhYPwz}
}

@inproceedings{
ye2025differential,
title={Differential Transformer},
author={Tianzhu Ye and Li Dong and Yuqing Xia and Yutao Sun and Yi Zhu and Gao Huang and Furu Wei},
booktitle={The Thirteenth International Conference on Learning Representations},
year={2025},
url={https://openreview.net/forum?id=OvoCm1gGhN}
}

@inproceedings{
gu2024mamba,
title={Mamba: Linear-Time Sequence Modeling with Selective State Spaces},
author={Albert Gu and Tri Dao},
booktitle={First Conference on Language Modeling},
year={2024},
url={https://openreview.net/forum?id=tEYskw1VY2}
}

@misc{olmo20252olmo2furious,
      title={2 OLMo 2 Furious}, 
      author={Team OLMo and Pete Walsh and Luca Soldaini and Dirk Groeneveld and Kyle Lo and Shane Arora and Akshita Bhagia and Yuling Gu and Shengyi Huang and Matt Jordan and Nathan Lambert and Dustin Schwenk and Oyvind Tafjord and Taira Anderson and David Atkinson and Faeze Brahman and Christopher Clark and Pradeep Dasigi and Nouha Dziri and Allyson Ettinger and Michal Guerquin and David Heineman and Hamish Ivison and Pang Wei Koh and Jiacheng Liu and Saumya Malik and William Merrill and Lester James V. Miranda and Jacob Morrison and Tyler Murray and Crystal Nam and Jake Poznanski and Valentina Pyatkin and Aman Rangapur and Michael Schmitz and Sam Skjonsberg and David Wadden and Christopher Wilhelm and Michael Wilson and Luke Zettlemoyer and Ali Farhadi and Noah A. Smith and Hannaneh Hajishirzi},
      year={2025},
      eprint={2501.00656},
      archivePrefix={arXiv},
      primaryClass={cs.CL},
      url={https://arxiv.org/abs/2501.00656}, 
}

@inproceedings{
hu2024minicpm,
title={Mini{CPM}: Unveiling the Potential of Small Language Models with Scalable Training Strategies},
author={Shengding Hu and Yuge Tu and Xu Han and Ganqu Cui and Chaoqun He and Weilin Zhao and Xiang Long and Zhi Zheng and Yewei Fang and Yuxiang Huang and Xinrong Zhang and Zhen Leng Thai and Chongyi Wang and Yuan Yao and Chenyang Zhao and Jie Zhou and Jie Cai and Zhongwu Zhai and Ning Ding and Chao Jia and Guoyang Zeng and dahai li and Zhiyuan Liu and Maosong Sun},
booktitle={First Conference on Language Modeling},
year={2024},
url={https://openreview.net/forum?id=3X2L2TFr0f}
}

@inproceedings{mixtape,
  author    = {Yang, Zhilin and Luong, Thang and Salakhutdinov, Russ R and Le, Quoc V},
  booktitle = {Advances in Neural Information Processing Systems},
  editor    = {H. Wallach and H. Larochelle and A. Beygelzimer and F. d\textquotesingle Alch\'{e}-Buc and E. Fox and R. Garnett},
  pages     = {},
  publisher = {Curran Associates, Inc.},
  title     = {Mixtape: Breaking the Softmax Bottleneck Efficiently},
  url       = {https://proceedings.neurips.cc/paper_files/paper/2019/file/512fc3c5227f637e41437c999a2d3169-Paper.pdf},
  volume    = {32},
  year      = {2019}
}

@inproceedings{petty-etal-2024-impact,
  title     = {The Impact of Depth on Compositional Generalization in Transformer Language Models},
  author    = {Petty, Jackson  and
               Steenkiste, Sjoerd  and
               Dasgupta, Ishita  and
               Sha, Fei  and
               Garrette, Dan  and
               Linzen, Tal},
  editor    = {Duh, Kevin  and
               Gomez, Helena  and
               Bethard, Steven},
  booktitle = {Proceedings of the 2024 Conference of the North American Chapter of the Association for Computational Linguistics: Human Language Technologies (Volume 1: Long Papers)},
  month     = jun,
  year      = {2024},
  address   = {Mexico City, Mexico},
  publisher = {Association for Computational Linguistics},
  url       = {https://aclanthology.org/2024.naacl-long.402/},
  doi       = {10.18653/v1/2024.naacl-long.402},
  pages     = {7239--7252}
}

@misc{allal2025smollm2smolgoesbig,
  title         = {SmolLM2: When Smol Goes Big -- Data-Centric Training of a Small Language Model},
  author        = {Loubna Ben Allal and Anton Lozhkov and Elie Bakouch and Gabriel Martín Blázquez and Guilherme Penedo and Lewis Tunstall and Andrés Marafioti and Hynek Kydlíček and Agustín Piqueres Lajarín and Vaibhav Srivastav and Joshua Lochner and Caleb Fahlgren and Xuan-Son Nguyen and Clémentine Fourrier and Ben Burtenshaw and Hugo Larcher and Haojun Zhao and Cyril Zakka and Mathieu Morlon and Colin Raffel and Leandro von Werra and Thomas Wolf},
  year          = {2025},
  eprint        = {2502.02737},
  archiveprefix = {arXiv},
  primaryclass  = {cs.CL},
  url           = {https://arxiv.org/abs/2502.02737}
}

@inproceedings{fineweb,
  title     = {The FineWeb Datasets: Decanting the Web for the Finest Text Data at Scale},
  author    = {Guilherme Penedo and Hynek Kydl{\'\i}{\v{c}}ek and Loubna Ben allal and Anton Lozhkov and Margaret Mitchell and Colin Raffel and Leandro Von Werra and Thomas Wolf},
  booktitle = {The Thirty-eight Conference on Neural Information Processing Systems Datasets and Benchmarks Track},
  year      = {2024},
  url       = {https://openreview.net/forum?id=n6SCkn2QaG}
}

@article{llama3modelcard,
  title  = {Llama 3 Model Card},
  author = {AI@Meta},
  year   = {2024},
  url    = {https://github.com/meta-llama/llama3/blob/main/MODEL_CARD.md}
}

@misc{qwen3technicalreport,
  title         = {Qwen3 Technical Report},
  author        = {Qwen Team},
  year          = {2025},
  eprint        = {2505.09388},
  archiveprefix = {arXiv},
  primaryclass  = {cs.CL},
  url           = {https://arxiv.org/abs/2505.09388}
}

@inproceedings{
Yun2020Are,
title={Are Transformers universal approximators of sequence-to-sequence functions?},
author={Chulhee Yun and Srinadh Bhojanapalli and Ankit Singh Rawat and Sashank Reddi and Sanjiv Kumar},
booktitle={International Conference on Learning Representations},
year={2020},
url={https://openreview.net/forum?id=ByxRM0Ntvr}
}

@inproceedings{godey2024why,
  title     = {Why do small language models underperform? Studying Language Model Saturation via the Softmax Bottleneck},
  author    = {Nathan Godey and {\'E}ric Villemonte de la Clergerie and Beno{\^\i}t Sagot},
  booktitle = {First Conference on Language Modeling},
  year      = {2024},
  url       = {https://openreview.net/forum?id=MoitXWlXcS}
}

@inproceedings{finlayson2024closing,
  title     = {Closing the Curious Case of Neural Text Degeneration},
  author    = {Matthew Finlayson and John Hewitt and Alexander Koller and Swabha Swayamdipta and Ashish Sabharwal},
  booktitle = {The Twelfth International Conference on Learning Representations},
  year      = {2024},
  url       = {https://openreview.net/forum?id=dONpC9GL1o}
}

@inproceedings{mircea2024gradient,
  title     = {Gradient Dissent in Language Model Training and Saturation},
  author    = {Andrei Mircea and Ekaterina Lobacheva and Irina Rish},
  booktitle = {High-dimensional Learning Dynamics 2024: The Emergence of Structure and Reasoning},
  year      = {2024},
  url       = {https://openreview.net/forum?id=tJj3psv9nm}
}

@inproceedings{grivas-etal-2022-low,
  title     = {Low-Rank Softmax Can Have Unargmaxable Classes in Theory but Rarely in Practice},
  author    = {Grivas, Andreas  and
               Bogoychev, Nikolay  and
               Lopez, Adam},
  editor    = {Muresan, Smaranda  and
               Nakov, Preslav  and
               Villavicencio, Aline},
  booktitle = {Proceedings of the 60th Annual Meeting of the Association for Computational Linguistics (Volume 1: Long Papers)},
  month     = may,
  year      = {2022},
  address   = {Dublin, Ireland},
  publisher = {Association for Computational Linguistics},
  url       = {https://aclanthology.org/2022.acl-long.465/},
  doi       = {10.18653/v1/2022.acl-long.465},
  pages     = {6738--6758}
}

@inproceedings{pmlr-v97-ganea19a,
  title     = {Breaking the Softmax Bottleneck via Learnable Monotonic Pointwise Non-linearities},
  author    = {Ganea, Octavian and Gelly, Sylvain and Becigneul, Gary and Severyn, Aliaksei},
  booktitle = {Proceedings of the 36th International Conference on Machine Learning},
  pages     = {2073--2082},
  year      = {2019},
  editor    = {Chaudhuri, Kamalika and Salakhutdinov, Ruslan},
  volume    = {97},
  series    = {Proceedings of Machine Learning Research},
  month     = {09--15 Jun},
  publisher = {PMLR},
  url       = {https://proceedings.mlr.press/v97/ganea19a.html}
}

@inproceedings{gao2018representation,
  title     = {Representation Degeneration Problem in Training Natural Language Generation Models},
  author    = {Jun Gao and Di He and Xu Tan and Tao Qin and Liwei Wang and Tieyan Liu},
  booktitle = {International Conference on Learning Representations},
  year      = {2019},
  url       = {https://openreview.net/forum?id=SkEYojRqtm}
}

@article{radford2019language,
  title  = {Language Models are Unsupervised Multitask Learners},
  author = {Radford, Alec and Wu, Jeff and Child, Rewon and Luan, David and Amodei, Dario and Sutskever, Ilya},
  year   = {2019}
}

@inproceedings{biderman2023pythia,
  title        = {Pythia: A suite for analyzing large language models across training and scaling},
  author       = {Biderman, Stella and Schoelkopf, Hailey and Anthony, Quentin Gregory and Bradley, Herbie and O’Brien, Kyle and Hallahan, Eric and Khan, Mohammad Aflah and Purohit, Shivanshu and Prashanth, USVSN Sai and Raff, Edward and others},
  booktitle    = {International Conference on Machine Learning},
  pages        = {2397--2430},
  year         = {2023},
  organization = {PMLR}
}

@misc{chinchilla_scaling,
  title         = {Training Compute-Optimal Large Language Models},
  author        = {Jordan Hoffmann and Sebastian Borgeaud and Arthur Mensch and Elena Buchatskaya and Trevor Cai and Eliza Rutherford and Diego de Las Casas and Lisa Anne Hendricks and Johannes Welbl and Aidan Clark and Tom Hennigan and Eric Noland and Katie Millican and George van den Driessche and Bogdan Damoc and Aurelia Guy and Simon Osindero and Karen Simonyan and Erich Elsen and Jack W. Rae and Oriol Vinyals and Laurent Sifre},
  year          = {2022},
  eprint        = {2203.15556},
  archiveprefix = {arXiv},
  primaryclass  = {cs.CL}
}

@inproceedings{chang-mccallum-2022-softmax,
  title     = {Softmax Bottleneck Makes Language Models Unable to Represent Multi-mode Word Distributions},
  author    = {Chang, Haw-Shiuan  and
               McCallum, Andrew},
  editor    = {Muresan, Smaranda  and
               Nakov, Preslav  and
               Villavicencio, Aline},
  booktitle = {Proceedings of the 60th Annual Meeting of the Association for Computational Linguistics (Volume 1: Long Papers)},
  month     = may,
  year      = {2022},
  address   = {Dublin, Ireland},
  publisher = {Association for Computational Linguistics},
  url       = {https://aclanthology.org/2022.acl-long.554},
  doi       = {10.18653/v1/2022.acl-long.554},
  pages     = {8048--8073}
}

@inproceedings{sigsoftmax,
  author    = {Kanai, Sekitoshi and Fujiwara, Yasuhiro and Yamanaka, Yuki and Adachi, Shuichi},
  booktitle = {Advances in Neural Information Processing Systems},
  editor    = {S. Bengio and H. Wallach and H. Larochelle and K. Grauman and N. Cesa-Bianchi and R. Garnett},
  pages     = {},
  publisher = {Curran Associates, Inc.},
  title     = {Sigsoftmax: Reanalysis of the Softmax Bottleneck},
  url       = {https://proceedings.neurips.cc/paper_files/paper/2018/file/9dcb88e0137649590b755372b040afad-Paper.pdf},
  volume    = {31},
  year      = {2018}
}

@inproceedings{softmax_bottleneck,
  title     = {Breaking the Softmax Bottleneck: A High-Rank {RNN} Language Model},
  author    = {Zhilin Yang and Zihang Dai and Ruslan Salakhutdinov and William W. Cohen},
  booktitle = {International Conference on Learning Representations},
  year      = {2018},
  url       = {https://openreview.net/forum?id=HkwZSG-CZ}
}

@article{kaplan_scaling,
  title   = {Scaling Laws for Neural Language Models},
  author  = {Jared Kaplan and Sam McCandlish and Tom Henighan and Tom B. Brown and Benjamin Chess and Rewon Child and Scott Gray and Alec Radford and Jeff Wu and Dario Amodei},
  journal = {ArXiv},
  year    = {2020},
  volume  = {abs/2001.08361},
  url     = {https://api.semanticscholar.org/CorpusID:210861095}
}

@inproceedings{lambada,
  author    = {Paperno, Denis  and  Kruszewski, Germ\'{a}n  and  Lazaridou,
               Angeliki  and  Pham, Ngoc Quan  and  Bernardi, Raffaella  and  Pezzelle,
               Sandro  and  Baroni, Marco  and  Boleda, Gemma  and  Fernandez, Raquel},
  title     = {The {LAMBADA} dataset: Word prediction requiring a broad
               discourse context},
  booktitle = {Proceedings of the 54th Annual Meeting of the Association for
               Computational Linguistics (Volume 1: Long Papers)},
  month     = {August},
  year      = {2016},
  address   = {Berlin, Germany},
  publisher = {Association for Computational Linguistics},
  pages     = {1525--1534},
  url       = {http://www.aclweb.org/anthology/P16-1144}
}

@misc{gao2020pile,
  title         = {The Pile: An 800GB Dataset of Diverse Text for Language Modeling},
  author        = {Leo Gao and Stella Biderman and Sid Black and Laurence Golding and Travis Hoppe and Charles Foster and Jason Phang and Horace He and Anish Thite and Noa Nabeshima and Shawn Presser and Connor Leahy},
  year          = {2020},
  eprint        = {2101.00027},
  archiveprefix = {arXiv},
  primaryclass  = {cs.CL}
}
\bibliographystyle{colm2026_conference}

\clearpage
\appendix
\crefalias{section}{appendix}
\crefalias{subsection}{appendix}

\section{Theoretical Analysis}\label{sec:theory}

\subsection{Setup: Notation and Loss Formulation}\label{sec:theory:setup}

We express the language modeling objective in a matrix form that will allow us to analyze both expressivity and gradient flow through the LM head.

We model language as a probability distribution $P$ over token sequences $(x_t)_{t=1}^L$, where tokens take values in a vocabulary of size $V$. We observe a dataset
\[
X = (x_{s,t})_{s=1,\dots,S;\,t=1,\dots,L_s}
\]
sampled from $P$, with $S$ the number of sampled sequences and $l_s$ the length of sequence $s \in [1, S]$. We denote $T=\sum_s l_s$ the total number of tokens. We decompose a language model as a backbone $\varphi_\theta$ that maps contexts to $\mathbb{R}^D$, and a linear layer $W_\theta\in\mathbb{R}^{V\times D}$. The standard maximum likelihood objective for a language model $(\varphi_\theta, W_\theta)$ can then be written as:
\begin{equation}\label{eq:nll}
\mathcal{L}(\theta, X)
=
-\frac{1}{T}
\sum_{s,t}
\big\langle \mathbf{1}_{x_{s,t}},
\log \sigma(W_\theta \varphi_\theta(\mathbf{x}_{s,<t}))
\big\rangle\;\;,
\end{equation}
where $\langle. ,.\rangle$ is the usual dot product, $\mathbf{1}_i$ is a one-hot vector in $\mathbb{R}^V$ and $\sigma$ denotes the softmax. The $\log$ function is taken element-wise.

Intuitively, in Transformers, $W_\theta\in\mathbb{R}^{V\times D}$ is the LM head projection and $\varphi_\theta$ is the rest of the network (i.e. below the LM head).

\paragraph{Fully Expressive Representations}
We assume that $\varphi_\theta$ is deterministic and sufficiently expressive so that each distinct context $c = \mathbf{x}_{s,<t}$ can be assigned any vector \mbox{$h_c = \varphi_\theta(c) \in\mathbb{R}^D$} without loss of generality. Under this assumption, optimization over $\theta$ is equivalent to optimization over parameters $\{h_c\}$ and $W_\theta$.

We rewrite the conventional log-likelihood loss in matrix notation, similar to~\citet{softmax_bottleneck}. 
Let $\mathcal{C}$ be the set of distinct contexts observed in $X$, with $|\mathcal{C}|=C$, and the next-token count vector for context $c\in\mathcal{C}$ be:
\begin{equation}
N_c
=
\sum_{s,t}
\delta_{c=\mathbf{x}_{s,<t}}\mathbf{1}_{x_{s,t}}
\in\mathbb{N}^V\;\;.
\end{equation}
where $\delta_{b} \in \{0,1\}$ is an indicator for boolean expression $b$.

Stacking these into a matrix $N\in\mathbb{R}^{C\times V}$ and letting \mbox{$H_\theta\in\mathbb{R}^{C\times D}$} denote the matrix of context representations (i.e. the representation before the LM head), the loss can be written compactly as:
\begin{equation}\label{eq:loss-matrix}
\mathcal{L}
=
-\frac{1}{T}
\langle
N,\log\sigma(H_\theta {W_\theta}^\top)
\rangle_F\;\;.
\end{equation}
where $\langle.,.\rangle_F$ is the Frobenius scalar product on matrices defined by:
\[
\langle A, B \rangle_F := \sum_{ij} A_{ij}B_{ij}
\]

\subsection{Expressivity and the Softmax Bottleneck}

The classical expressivity view of the softmax bottleneck is that the LM head imposes a rank constraint on the model's output log-probabilities, which can prevent the model from matching the true data distribution. We show, however, that this rank constraint does not affect top-1 predictions, motivating a closer look at the \emph{optimization} consequences in \Cref{sec:theory:gradientdescent}.

First, we revisit the characterization of the optimal achievable likelihood under no rank constraints.

Informally, we are trying to match $N$ in $\log$ space using the hidden states $H_\theta$ and the LM head $W_\theta$. However, element-wise $\log$ is not defined for null entries, which can be expected to be observed often in $N$. To allow for a direct use of $\log$ on $N$ while avoiding complex notations, we take advantage of the limit of $x \rightarrow x\log x$ in $0^+$ and we extend the Frobenius scalar product on matrices $\langle .,. \rangle_F$ with 
\[
\langle X, \log X \rangle_F := \sum_{ij} \mathbf{1}_{X_{ij}}X_{ij}\log X_{ij}
\]

\begin{proposition}[Gibbs inequality]\label{prop:gibbs}
Let $\tilde{N}$ denote the row-normalized version of $N$. Then
\begin{equation}
\mathcal{L}
\ge
-\frac{1}{T}
\langle N,\log\tilde{N}\rangle_F,
\end{equation}
with equality if and only if
$\sigma(H_\theta {W_\theta}^\top)=\tilde{N}$.
\end{proposition}

This proposition identifies the matrix of empirical conditional distributions $\tilde{N}$ as the unconstrained optimum. However, the equality condition is typically unattainable due to rank constraints on $H_\theta {W_\theta}^\top$.

\begin{proposition}[\citealt{softmax_bottleneck}]\label{prop:ranks}
We have
\[
\rank(H_\theta {W_\theta}^\top)\le D\;\;,
\quad
\rank(\log\sigma(H_\theta {W_\theta}^\top))\le D+1\;\;.
\]
\end{proposition}

\begin{proof}
The first inequality is immediate. For the second, observe that for any matrix $A$, each row of $\log\sigma(A)$ equals $A_i - z_i\mathbf{1}$, where $z_i$ is the $\log$ of the softmax denominator. This implies that the row space is spanned by that of $A$ together with the all-ones vector, and thus has a rank bounded by $\rank(A) + 1$.
\end{proof}

\begin{corollary}[\citealt{softmax_bottleneck}]
We assume $\forall i,j, \tilde{N}_{ij} > 0$. The empirical minimum for \Cref{eq:loss-matrix} is attainable if and only if
\[
\rank(\log\tilde{N})\le D+1\;\;.
\]
\end{corollary}

This result formalizes the softmax bottleneck: even with perfect context representations, the output layer imposes a strict low-rank constraint on realizable log-probabilities. In other words, when $D \ll V$, the model simply cannot assign arbitrary probability distributions over the vocabulary, regardless of how the backbone is trained.

However, this expressivity limitation does not prevent the model from correctly identifying the \emph{most likely} next token and its probability up to arbitrary precision when $D \ge 2$:

\begin{proposition}[Proof in \Cref{proof:top1}]\label{prop:top1}
When $D \ge 2$, for all $\varepsilon > 0$ and $i \in [1, C]$, there exists $H \in \mathbb{R}^{C \times D}$ and $W \in \mathbb{R}^{V \times D}$ so that: 
\begin{equation}\label{eq:top1ineq}
\left| \sigma(H_iW^\top)_{i,{\argmax(\tilde{N}_i)}} - \tilde{N}_{i,\argmax(\tilde{N}_i)} \right| < \varepsilon\;\;.
\end{equation}
\end{proposition}

Therefore, even though the low-rank constraint limits the model expressivity, the set of reachable top-1 probabilities is not constrained, and the overall behavior of the language model is theoretically unaffected for greedy decoding. This means that the softmax bottleneck cannot be fully characterized as an expressivity issue alone, and motivates us to examine it through the lens of optimization dynamics.

\subsection{The Gradient Descent Viewpoint}\label{sec:theory:gradientdescent}

Our analysis shows how the softmax bottleneck affects \emph{training}, not just expressivity. The key insight is that backpropagating through the LM head compresses the gradient from $V$ dimensions down to at most $D$ dimensions, discarding the majority of the training signal. We formalize this by tracking how the logits evolve under gradient descent, and showing that the resulting update direction is fundamentally misaligned with the logit gradient whenever $D \ll V$.

Recall that the logits (i.e. the activations after the LM head projection, but before the softmax) are given by
\[
L_\theta^\tau = H_\theta^\tau (W_\theta^\tau)^\top\;\;,
\]
where $\tau$ denotes the gradient descent iteration. Understanding how the logits \mbox{$L_\theta^\tau \in \mathbb{R}^{C \times V}$} evolve is crucial, because the loss depends on $\theta$ only through the induced probability matrix \mbox{$P_\theta^\tau = \sigma(L_\theta^\tau) \in \mathbb{R}^{C \times V}$}.

\paragraph{Logit Dynamics Under Gradient Descent}
Let $P_\theta^\tau = \sigma(L_\theta^\tau)$, and define the empirical context frequency
\[
f_i := \frac{\sum_j N_{ij}}{T}\;\;.
\]
The LM head update with learning rate $\eta$ is given by
\[
W_\theta^{\tau+1}
= W_\theta^\tau
- \eta \cdot \,(P_\theta^\tau - \tilde{N})^\top \operatorname{diag}(f) H_\theta^\tau\;\;.
\]

\paragraph{Effective update direction.}
For convenience, we define the rescaled logit update
\[
\Delta_\theta^\tau = \frac{L_\theta^{\tau+1} - L_\theta^\tau}{\eta}\;\;.
\]
By construction, $\Delta_\theta^\tau$ is a sum of two matrices whose ranks are upper-bounded by $D$, which implies that
\begin{equation}\label{eq:delta}
\rank(\Delta_\theta^\tau) \leq 2D\;\;.
\end{equation}
As a consequence, however large $V$ may be, the actual global logit update is at most rank $2D$. By contrast, if the logits were optimized directly, i.e. if we performed gradient descent on the logit matrix taken as a parameter, the first-order optimal update direction in the limit $\eta \to 0$ would be given by the logit gradient
\begin{equation}\label{eq:logit_grad}
\nabla_L \mathcal{L}(\theta^\tau, X) = \operatorname{diag}(f) (P_\theta^\tau - \tilde{N})\;\;.
\end{equation}
The central question is therefore: \emph{under which conditions can the actual update direction $\Delta_\theta^\tau$ induced by gradient descent (\Cref{eq:delta}) match the logit-optimal direction (\Cref{eq:logit_grad})?} Answering this question requires analyzing the structure and rank properties of $(P_\theta^\tau - \tilde{N})$.

\paragraph{Rank of the Prediction Error}
We begin by identifying contexts for which the empirical next-token distribution is degenerate.

\begin{proposition}[Proof in \Cref{proof:rankPN}]\label{prop:rankPN}
Consider the set of contexts
\[
\mathcal{C}^u
= \{\mathbf{x}_{s,<t} \mid \tilde{N}_{x_{s,t}} = \mathbf{1}_{x_{s,t}}\}
\]
for which the empirical next-token probability is concentrated on a single token. Let
\[
\mathcal{C}^{uu}
= \{ x_{s,t} \mid \mathbf{x}_{s,<t} \in \mathcal{C}^u \}
\]
denote the set of corresponding unique next tokens. For any \mbox{$P \in (0,1)^{C \times V}$}, we have
\[
\rank(P - \tilde{N}) \geq \min(|\mathcal{C}^{uu}|, V-1) := \mathcal{D}^u(N)\;\;.
\]
\end{proposition}

Informally, this says that as long as many distinct tokens each appear as the unique observed continuation of some context, the prediction error matrix $(P - \tilde{N})$ has rank close to $V$ for any $P$.

In natural language, it is reasonable to expect that most tokens appear at least once in a context where the observed continuation is unique in the training set, especially for long sequences. Some continuations are even constrained by the language distribution, e.g., \emph{earth} following \emph{down-to-}. This motivates the hypothesis that in practical language modeling setups $|\mathcal{C}^{uu}| \simeq V$, or at least $|\mathcal{C}^{uu}| \gg D$. We corroborate these arguments with experimental observations in \Cref{app:valid_assumptions}.

\paragraph{Mismatch Between Ideal and Realizable Updates}
Combining the rank lower bound on $(P^\tau_\theta - \tilde{N})$ from \Cref{prop:rankPN} with the rank upper bound on the logit update direction $\Delta_\theta^\tau$ from \Cref{eq:delta}, we can now quantify the inevitable gap between the actual update and the logit gradient. We measure this gap in terms of minimal Frobenius distance. In practice, the observed $\Delta_\theta^\tau$ and $(P_\theta^\tau - \tilde{N})$ could have different norms; nevertheless, a general lower bound on the Frobenius distance holds through any rescaling of the logit update. Hence, the following result implies that the \emph{directions} of the logit gradient and the logit update cannot be aligned.

\begin{proposition}[Update direction residual]\label{prop:upd_res}
If \mbox{$\mathcal{D}^u(N) > D$}, then
\[
\|\Delta_\theta^\tau - (P_\theta^\tau - \tilde{N})\|_F
> \sqrt{\sum_{i=2D+1}^{\mathcal{D}^u(N)} \varsigma_i^2}
> 0\;\;,
\]
where $||.||_F$ is the Frobenius norm and $\varsigma_i$ are the singular values of $(P_\theta^\tau - \tilde{N})$ in decreasing order.
\end{proposition}

\begin{proof}
This is a direct application of the Eckart--Young--Mirsky theorem.
\end{proof}

In other words, whenever the logit gradient has intrinsic rank exceeding $2D$, which \Cref{prop:rankPN} shows to be the typical case, no gradient descent step through a rank-$D$ LM head can match it. The missing components correspond to the tail singular values of $(P_\theta^\tau - \tilde{N})$, and the bound worsens as $D$ decreases. In that sense, the LM head is a gradient bottleneck.

\subsubsection{Stochastic gradient descent}

One might hope that stochastic gradient descent (SGD) mitigates this issue: mini-batches see only a subset of contexts, so the effective $\tilde{N}^\mathcal{B}$ within a batch could in principle have lower rank, relaxing the bottleneck. This is not the case though. The high-rank structure of the logit gradient persists in the SGD setting, particularly as the model approaches convergence.

\paragraph{Notation.} Consider a batch $\mathcal{B} \in \mathcal{P}([1, S])$, where we recall that $S$ is the number of sequences in our observed dataset $X$, and $\mathcal{P}$ is the power set operator. This batch contains contexts $\mathcal{C_B} \in \mathcal{P}(\mathcal{C})$. 

We extract all batch-level objects that are relevant to the mini-batch optimization problem. Let $N^\mathcal{B}$ and $\tilde{N}^\mathcal{B}$ be the in-batch context-token count matrix and its row-normalized version, and  $(P^\tau_\theta)_\mathcal{B}$ the matrix extracted from $P^\tau_\theta$ by retaining rows corresponding to contexts in $\mathcal{C_B}$. 

\begin{proposition}[Proof in \Cref{proof:sgd}]\label{prop:sgd}
    Let $\mathcal{C_B}^u$ be the set of contexts $c$ that:
    \begin{itemize}
        \item have a unique and distinct next token $x_c$ in the batch $\mathcal{B}$ while having several observed next tokens in the dataset $X$;
        \item are such that $(\tilde{N}_{cx_{c'}})_{c, c'\in\mathcal{C}_\mathcal{B}^{u}}$ is analogous to the adjacency matrix of a connected graph.
    \end{itemize}
    There exists $\varepsilon_\mathcal{B} > 0$ such that if
    \begin{equation}\label{eq:eps_condition}
    \max_{c \in \mathcal{C_B}^{u}}||(P^\tau_{\theta})_c - \tilde{N}_c||_\infty < \varepsilon_\mathcal{B}
    \end{equation}
    then
    \[
    \rank ((P^\tau_\theta)_\mathcal{B} - \tilde{N}^\mathcal{B}) \geq \mathcal{D^B}(N) := \min(|\mathcal{C_B}^{u}|, V-1)\;\;.
    \]
\end{proposition}

\Cref{prop:sgd} indicates that when the model predictions $P^\tau_\theta$ get closer to the empirical data distribution $\tilde{N}$, the first-order optimal logit update direction for batch $\mathcal{B}$ is high-rank and the actual logits update is therefore suboptimal, by the same mechanism as in \Cref{prop:upd_res}. Informally, the in-batch empirical probability $\tilde{N}^\mathcal{B}_c$ is likely to be sparse even in contexts $c$ where the global empirical probability $\tilde{N}_c$ is dense. However, $(P^\tau_\theta)_c$ is not batch-dependent, so when $(P^\tau_\theta)_c$ globally converges towards $\tilde{N}$, it is still compared to a sparse $\tilde{N}^\mathcal{B}_c$ at batch level. We can thus use similar arguments to \Cref{prop:rankPN} to conclude about the rank of the logit gradient.

\begin{corollary}
\label{cor:sgd}
    We assume that the model's parameters $\theta$ are such that 
    $$
    \forall \mathcal{B}, |\mathcal{C}_{\mathcal{B}}^{u}| \geq V -1\;\;.
    $$
    Then, we have:
    $$
    ||\Delta_{\mathcal{B}\theta}^\tau - ((P^\tau_\theta)_\mathcal{B} - \tilde{N}^\mathcal{B})||_F > \sqrt{\sum_{i=2D+1}^{V-1}\varsigma_i^2} > 0\;\;,
    $$
    where $\varsigma_i$ are the singular values of $(P^\tau_\theta)_\mathcal{B} - \tilde{N}^\mathcal{B}$ in decreasing order.
\end{corollary}

Taken together, \Cref{prop:sgd} and \Cref{cor:sgd} show that the gradient bottleneck is not an artifact of full-batch training: near convergence, even mini-batch gradients are high-rank, and the rank-$D$ LM head remains a fundamental obstacle.

The weakness of the connectivity condition in \Cref{prop:sgd} is justified for natural language because special and frequent tokens should often have non-null probability and thus behave similarly to high-degree vertices. Many next-token probability distributions should also have relatively high entropy and increase connectivity. We provide empirical elements to back these arguments in \Cref{app:valid_assumptions}. In practice, for batches containing $10^5-10^7$ tokens, the number of positive-entropy contexts that appear only once should not be greatly inferior to $V$.

\paragraph{Alternative LM Heads}
One might also wonder whether replacing the standard softmax with a more expressive output layer, as proposed in past work to address the expressivity bottleneck, resolves the optimization issue as well. We show that it does not. Past work~\citep{softmax_bottleneck,mixtape,pmlr-v97-ganea19a} has proposed alternatives to the softmax layer to improve the expressivity of models in terms of output probabilities. These alternatives break down the LM head into several components before recombining outputs, which theoretically boosts the maximum rank of the output log-probabilities. However, the proposed alternatives theoretically fail to address the softmax bottleneck from the optimization perspective. Indeed, if we consider any mapping \mbox{$g_\theta(H_\theta) = L_\theta$}, then
\begin{equation}\label{eq:jacobian}
\nabla_H \mathcal{L}(\theta^\tau, X) = \operatorname{diag}(f) (P_\theta^\tau - \tilde{N}) J_g(H_\theta)\;\;,
\end{equation}
where the Jacobian $J_g(H_\theta)$ is rank $D$ at most. The optimization process is thus limited -- in the first-order -- by the same effect we described in the softmax case, regardless of the specific choice of $g$.

Nevertheless, even though $H_\theta$ never benefits from a complete view of the loss gradient in the logits, some functions $g$ may provide less lossy compression through $J_g$, potentially through pre-conditioning or regularization. Such alternatives would not target full expressivity, but rather aim to preserve gradient flow through LM heads. We leave the exploration of such methods for future work.

\paragraph{Empirical Validation of the Theory} 
The theoretical results above hint to the fact that logit gradients are near full-rank in practice. We verify this directly: we estimate the empirical ranks of the logits' gradients of the Pythia models~\citep{biderman2023pythia} in \Cref{fig:pythia_grad_rank}, by measuring them on data taken from the Pile~\citep{gao2020pile}. We obtain their empirical ranks by counting non-null ($>$1e-6) diagonal $R$ values in their $QR$ decomposition.

We also study the relevance of the bounds in \Cref{prop:upd_res,cor:sgd}. In \Cref{fig:svd_components}, we measure the number of SVD components required to approximate the logit gradient as the tolerated approximation error increases, for various batch sizes. Although 70\% of the gradient can be expressed with less than 1{,}000 components, almost exact approximation (up to 99.9\%) requires an increasing number of components as we add tokens in the batch, with up to 30{,}000 components for our maximal batch size (500{,}000 tokens). This confirms that the singular value spectrum is genuinely broad and
the theoretical bound of \Cref{prop:upd_res} is non-vacuous.

\begin{figure}[t]
    \centering
    \begin{subfigure}[t]{0.47\textwidth}
        \centering
        \includegraphics[width=\linewidth]{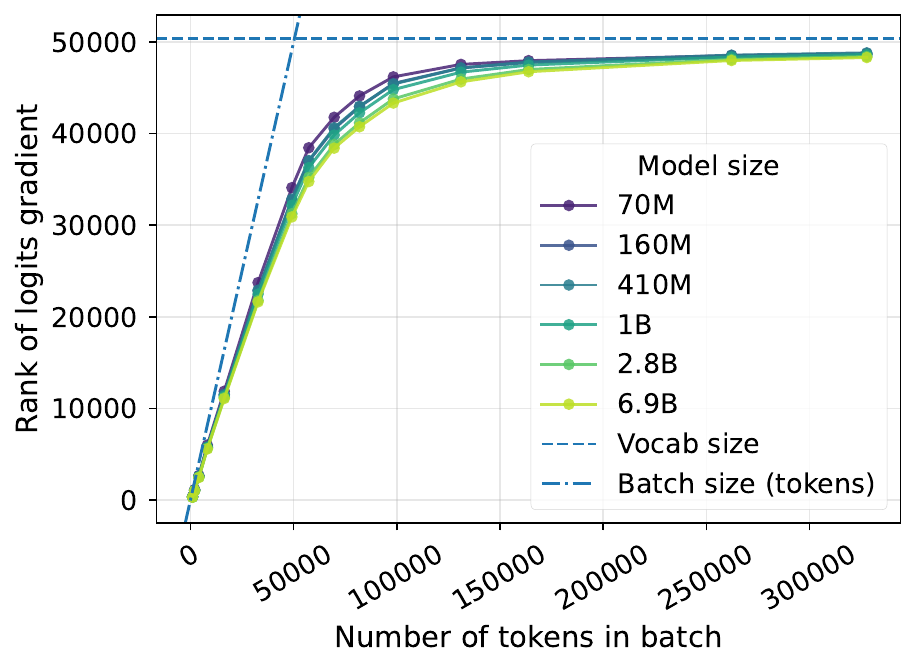}
    \caption{Empirical rank across Pythia models}\label{fig:pythia_grad_rank}
    \end{subfigure}
    \begin{subfigure}[t]{0.52\textwidth}
        \centering
        \includegraphics[width=\linewidth]{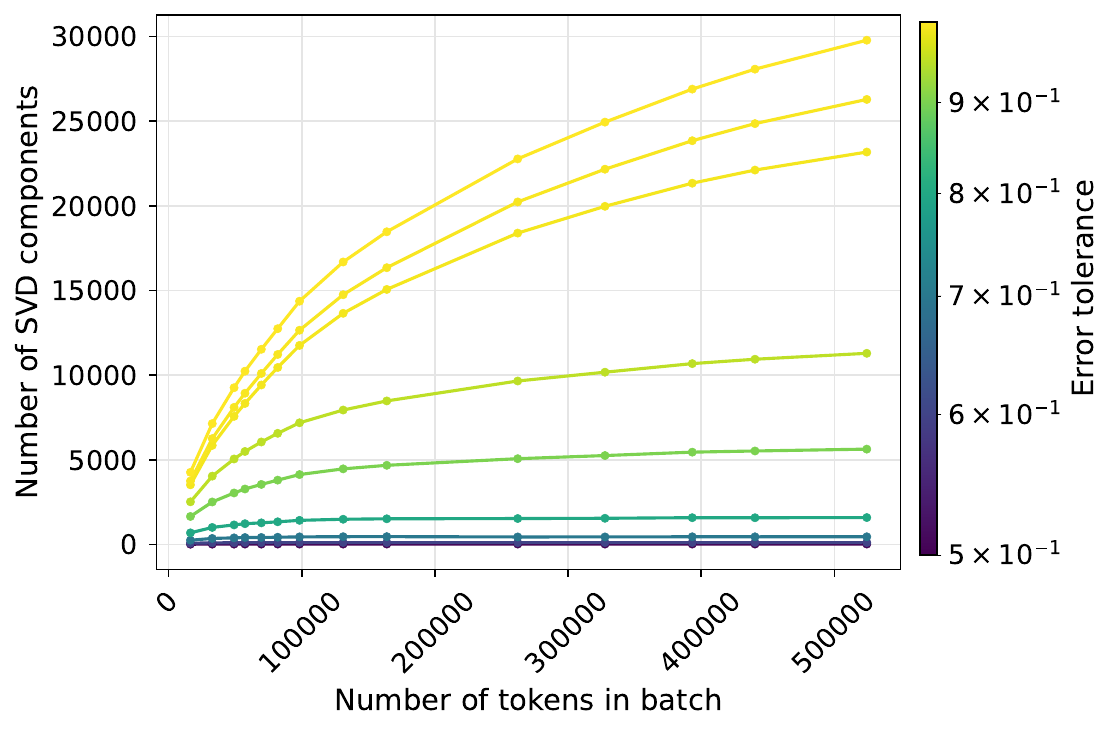}
    \caption{Effective rank analysis (Pythia-1B)}\label{fig:svd_components}
    \end{subfigure}
    \caption{
Empirical and effective rank of the logit gradient $\nabla_L \mathcal{L}$
as a function of batch size.
\textbf{(a)} Empirical rank across Pythia model sizes, measured on the
Pile~\citep{gao2020pile} via QR decomposition.
\textbf{(b)} Number of SVD components required to reconstruct $\nabla_L \mathcal{L}$
up to a given fraction of its total norm (color), for Pythia-1B.}
    \label{fig:rank_svd_qr}
\end{figure}

\section{Proofs}\label{app:proofs}

\subsection{\Cref{prop:top1}}\label{proof:top1}

\begin{proof}
If $D=2$, we can pick rows of $W$ as distinct points on the unit circle. We then set:
$$
H_i = \alpha_i W_{\argmax(\tilde{N}_i)}
$$
with $\alpha_i \ge 0$. We clearly have:
$$
\begin{cases}
    \lim_{\alpha_i \rightarrow 0} \sigma(H_iW^\top)_{i,\argmax(\tilde{N}_i)} = \frac{1}{V} \\
    \lim_{\alpha_i \rightarrow +\infty} \sigma(H_iW^\top)_{i,\argmax(\tilde{N}_i)} = 1
\end{cases}
$$
By definition, $\tilde{N}_{i,\argmax(\tilde{N}_i)} \in [\frac{1}{V}, 1]$. Hence, by continuity of the softmax function, for all $i \in [1, C]$ and $\varepsilon > 0$, there exists $\alpha_i$ that satisfies \Cref{eq:top1ineq}.

For $D > 2$, we can work on the first two dimensions of $H$ and $W$ and zero-out all other entries to reach the same conclusion.
\end{proof}

\subsection{\Cref{prop:rankPN}}\label{proof:rankPN}

\begin{proof}
We introduce $P^u$ and $\tilde{N}^u$, where we extract rows that correspond to $\mathcal{C}^u$ from $P$ and $\tilde{N}$. As $P_{ij} \in ]0, 1[$, we have:
    $$
    (P^u - \tilde{N}^u)_{ij} \begin{cases}
        <0\text{ if } j=w_{s, t} \\
        >0 \text{ else.}
    \end{cases}
    $$

    We also note that
    $$
    \sum_j (P^u - \tilde{N}^u)_{ij} = \sum_j (P - \tilde{N})_{ij} = 1-1 = 0
    $$

    Let us now have $P^{uu}$ and $\tilde{N}^{uu}$ by extracting columns from $P^u$ and $\tilde{N}^u$ that correspond to $\mathcal{C}^{uu}$. We thus notice that $\tilde{N}^{uu}$ corresponds to the identity $I_{|\mathcal{C}^{uu}|}$ up to a permutation of rows. We thus rearrange the rows of $P^{uu}$ and $\tilde{N}^{uu}$ so that $\tilde{N}^{uu} = I_{|\mathcal{C}^{uu}|}$ without loss of generality. This implies that
    $$
    (P^{uu} - \tilde{N}^{uu})_{ij} \begin{cases}
        <0\text{ if } i=j \\
        >0 \text{ else.}
    \end{cases}
    $$
    
    As we only removed columns where coefficients of $(P^u - \tilde{N}^u)$ were positive, we also have 
    $$
    \sum_j (P^{uu} - \tilde{N}^{uu})_{ij} \begin{cases}
        < 0\text{ if }|\mathcal{C}^{uu}| < V \\
        = 0\text{ if }|\mathcal{C}^{uu}| = V
    \end{cases}
    $$

    If $|\mathcal{C}^{uu}| = V$, then $(P^{uu} - \tilde{N}^{uu})$ is a square matrix without null coefficients, whose rows sum to $0$, and where only diagonal terms are negative. It is thus analogous to the Laplacian of a fully connected graph with $V$ vertices, and has rank $V-1$. Else, then $(P^{uu} - \tilde{N}^{uu})$ is strictly diagonally dominant, and is full-rank by Gershgorin's circle theorem. As $(P^{uu} - \tilde{N}^{uu})$ is a submatrix of $(P - \tilde{N})$, we have 
    $$
    \rank (P - \tilde{N}) \geq \rank (P^{uu} - \tilde{N}^{uu}) = \min(|\mathcal{C}^{uu}|, V-1)
    $$

\end{proof}
\subsection{\Cref{prop:sgd}}\label{proof:sgd}

\begin{proof}
By the definition of $\mathcal{C}_\mathcal{B}^u$, each context has a unique and distinct next token in batch $\mathcal{B}$ while having multiple possible next tokens in dataset $X$. Under condition \eqref{eq:eps_condition} with sufficiently small $\varepsilon_\mathcal{B}$, the submatrix of $(P^\tau_\theta - \tilde{N}^\mathcal{B})$ corresponding to rows $\mathcal{C}_\mathcal{B}^u$ preserves the sign pattern of $(\tilde{N} - \tilde{N}^\mathcal{B})$ on these rows.

Following the construction in \Cref{prop:rankPN}, we extract rows from $\mathcal{C}_\mathcal{B}^u$ and columns corresponding to the distinct unique batch tokens. Since these tokens are distinct by assumption, we obtain a square submatrix of size $|\mathcal{C}_\mathcal{B}^u| \times |\mathcal{C}_\mathcal{B}^u|$.

If $|\mathcal{C}_\mathcal{B}^u| = V$, the connectivity assumption ensures this submatrix is analogous to the Laplacian of a connected graph, yielding rank $V-1$.

If $|\mathcal{C}_\mathcal{B}^u| < V$, the same argument as in the proof of \Cref{prop:rankPN} establishes strict diagonal dominance, hence full rank $|\mathcal{C}_\mathcal{B}^u|$ by Gershgorin's circle theorem.

Therefore $\rank(P^\tau_\theta - \tilde{N}^\mathcal{B}) \geq \min(|\mathcal{C}_\mathcal{B}^u|, V-1)$.
\end{proof}

\section{Empirical Justifications for Assumptions}
\label{app:valid_assumptions}

This section reports basic experiments that help justify the assumptions made in \Cref{prop:rankPN} and \Cref{prop:sgd} for natural language data.

In \Cref{fig:uniques}, we compute basic statistics about unique tokens and contexts in batches of Fineweb documents. The results indicate that nearly all observed contexts are unique, and that most tokens appear at least once for batches of usual sizes (up to a few million tokens). This suggests that the bound in \Cref{prop:rankPN} should take high values, as unique contexts made of tokens spanning most of the vocabulary lead to bigger $\mathcal{C}^{uu}$ sets.

\begin{figure}[t]
    \centering
    \begin{subfigure}[t]{0.47\textwidth}
        \centering
        \includegraphics[width=\linewidth]{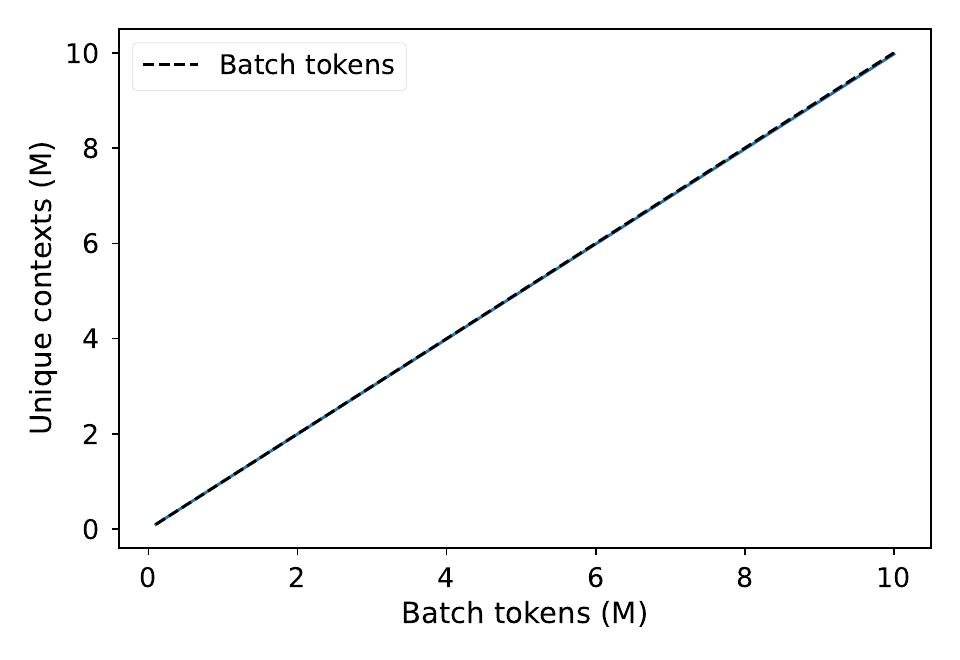}
    \caption{Unique contexts count}
    \end{subfigure}
    \begin{subfigure}[t]{0.47\textwidth}
        \centering
        \includegraphics[width=\linewidth]{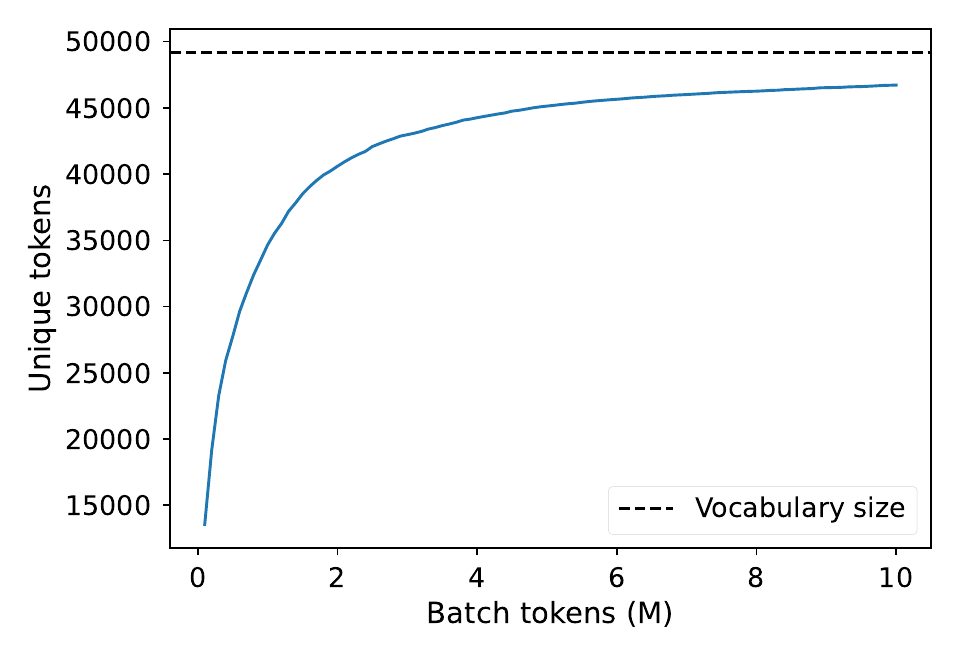}
    \caption{Unique tokens count}
    \end{subfigure}
    \caption{Unique tokens/contexts count as batch size increases. We use the SmolLM2 tokenizer on Fineweb documents.}
    \label{fig:uniques}
\end{figure}

In \Cref{fig:nt_ent}, we measure the next-token entropy distribution for Llama-3.1-8B on Fineweb, as a proxy of $\tilde{N}$. These results indicate that this distribution has relatively scattered support, with a non-negligible number of contexts resulting in entropy levels $>5$. Following the connection made in \Cref{prop:sgd} between $\tilde{N}$ and adjacency matrices, these contexts are likely to act as highly connected vertices, boosting the connectivity of the underlying graph.

\begin{figure}[tbh]
    \centering
    \includegraphics[width=0.6\linewidth]{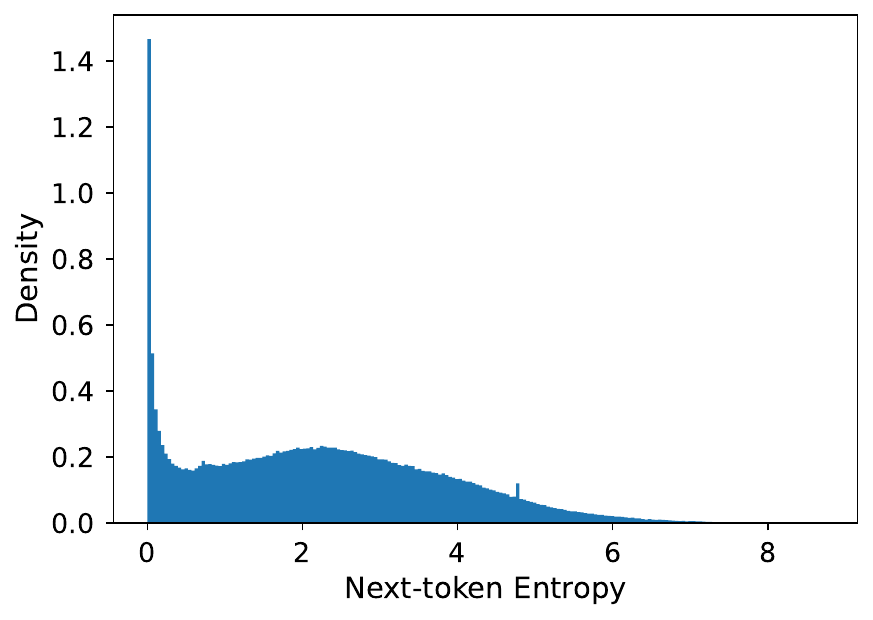}
    \caption{Next-token entropy distribution for Llama-3.1-8B on Fineweb documents.}
    \label{fig:nt_ent}
\end{figure}

\section{Analysis of the Gradient Compression}
\label{app:grad_comp}

In \Cref{fig:logit_grad_shape_l32b}, we report additional results for the experiment led in \Cref{ssec:exp:grad_comp} for Llama-3.1-8B and OLMo2-32B. We reach similar conclusions as for Llama-3.1-70B, namely that the main components of the gradients are smoothened by compression, and the tail of coefficients is noisier after projection.

\begin{figure*}[t!]
    \centering
    \begin{subfigure}[t]{0.47\textwidth}
        \centering
        \includegraphics[width=\textwidth]{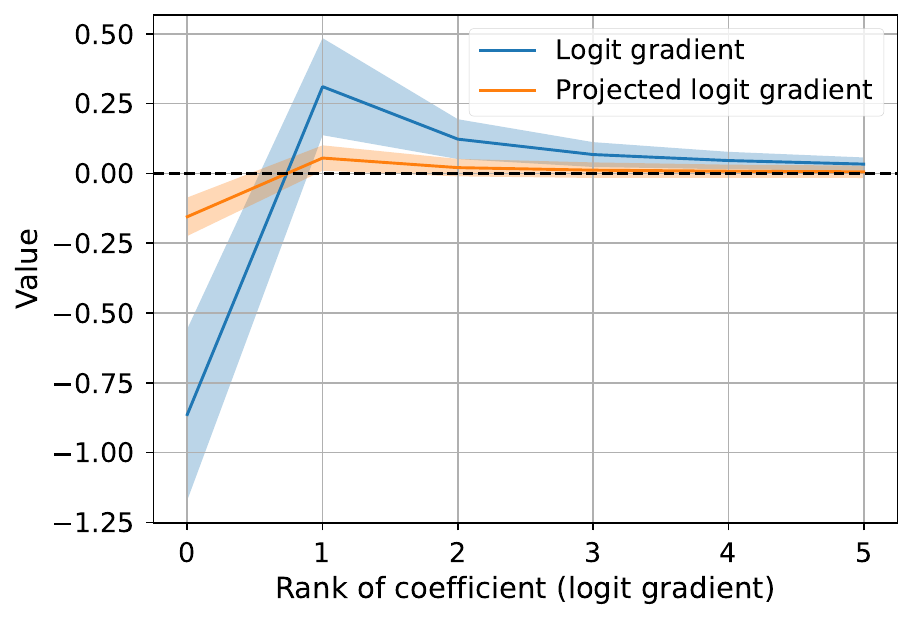}
        \caption{Llama-3.1-8B}
    \end{subfigure}%
    ~ 
    \begin{subfigure}[t]{0.47\textwidth}
        \centering
        \includegraphics[width=\textwidth]{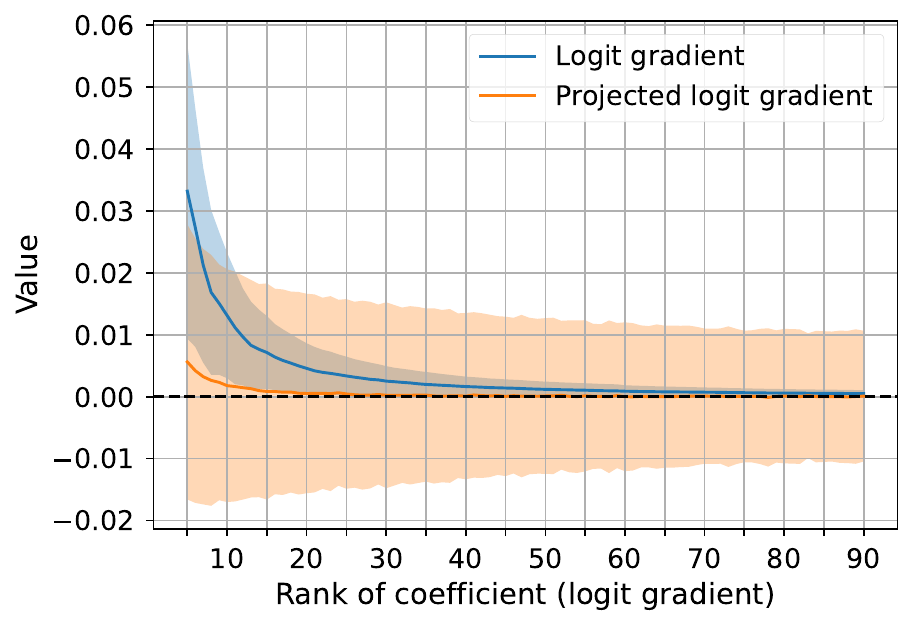}
        \caption{Llama-3.1-8B (continued)}
    \end{subfigure}
    \centering
    \begin{subfigure}[t]{0.47\textwidth}
        \centering
        \includegraphics[width=\textwidth]{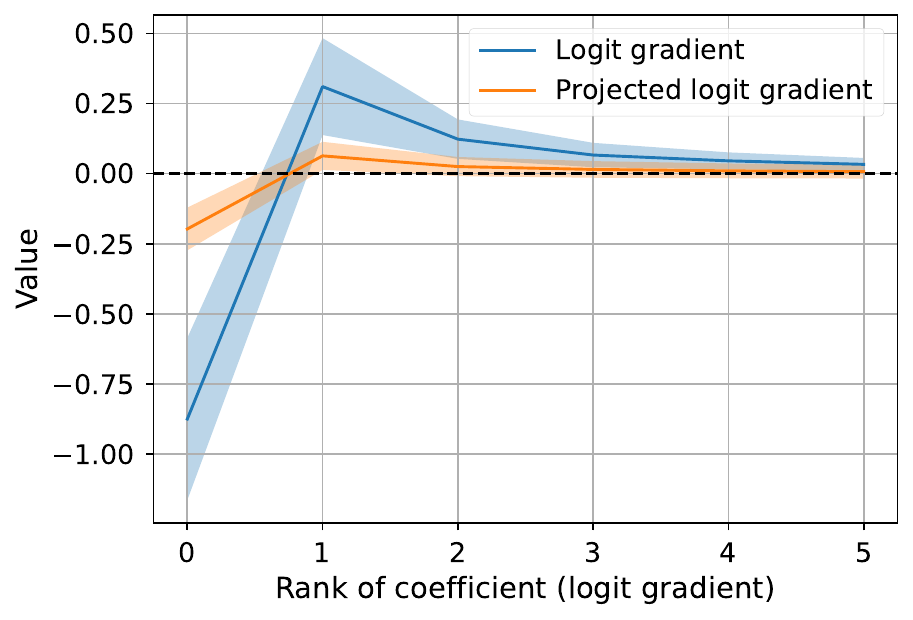}
        \caption{OLMo2-32B}
    \end{subfigure}%
    ~ 
    \begin{subfigure}[t]{0.47\textwidth}
        \centering
        \includegraphics[width=\textwidth]{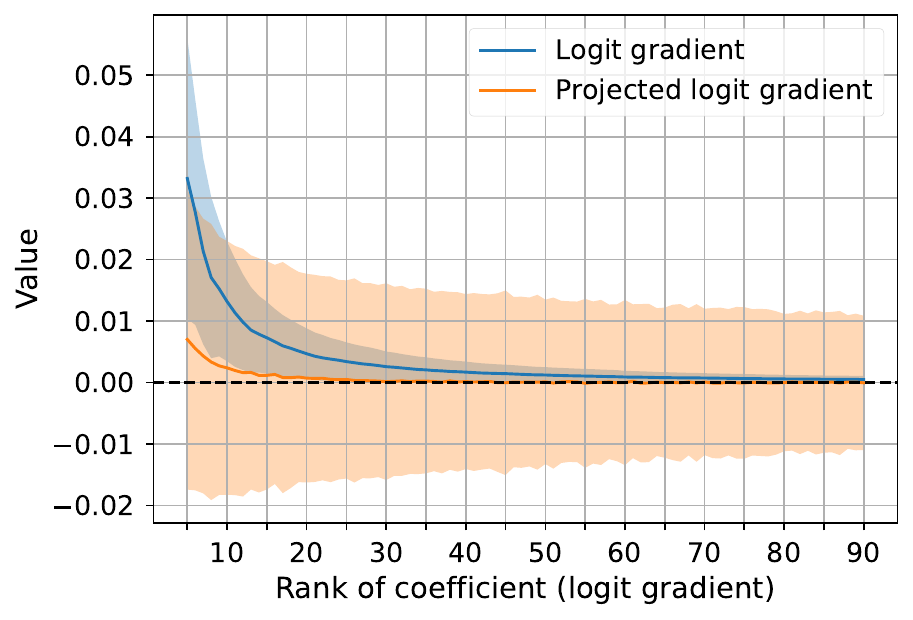}
        \caption{OLMo2-32B (continued)}
    \end{subfigure}
    \caption{Average full and projected logits gradient coefficients, sorted by absolute values of the full gradient. The gradients are taken for Llama-3.1-8B and OLMo2-32B on Fineweb documents. The plots on the left-hand side focus on the main coefficients, while the plots on the right-hand side zoom in on smaller coefficients. The standard deviation is reported using filled areas around the curves.}
    \label{fig:logit_grad_shape_l32b}
\end{figure*}
\FloatBarrier

\section{Training Dynamics}\label{app:training_dynamics}

In \Cref{fig:training_dynamics}, we report the same metrics as in \Cref{fig:gradient_compression} along training for OLMo-2-1B~\citep{olmo20252olmo2furious}. For each intermediate checkpoint, we compute the null space of $W_\theta^\top$, and we then measure both the fraction of projected norm and the cosine similarity between the original and projected gradients, using FineWeb samples. Overall, the metrics are stable throughout training, with two noticeable exceptions:
\begin{itemize}
    \item In the first few billion tokens, the compression slightly worsens to reach a plateau that it sustains throughout training;
    \item During the midtraining phase, a slight improvement in both compression metrics occurs. This coincides with both a change in the training data distribution and a learning rate decay to 0, which makes it unclear what caused this transition.
\end{itemize}
Even though this effect is minor, we leave further exploration for future work.

\begin{figure*}[t!]
    \centering
    \begin{subfigure}[t]{0.47\textwidth}
        \centering
        \includegraphics[width=\textwidth]{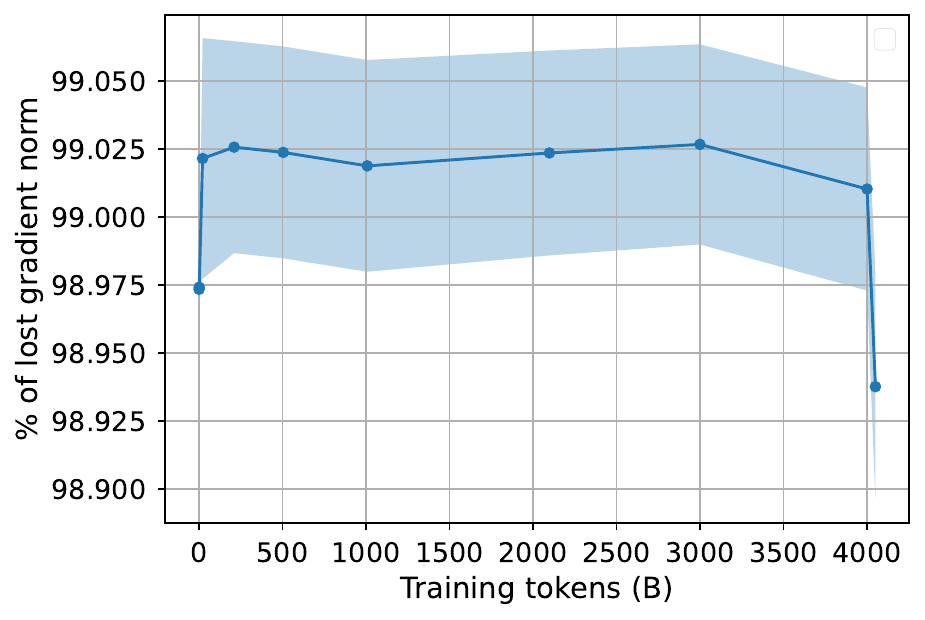}
        \caption{Fraction of projected gradient norm.}
    \end{subfigure}%
    ~ 
    \begin{subfigure}[t]{0.47\textwidth}
        \centering
        \includegraphics[width=\textwidth]{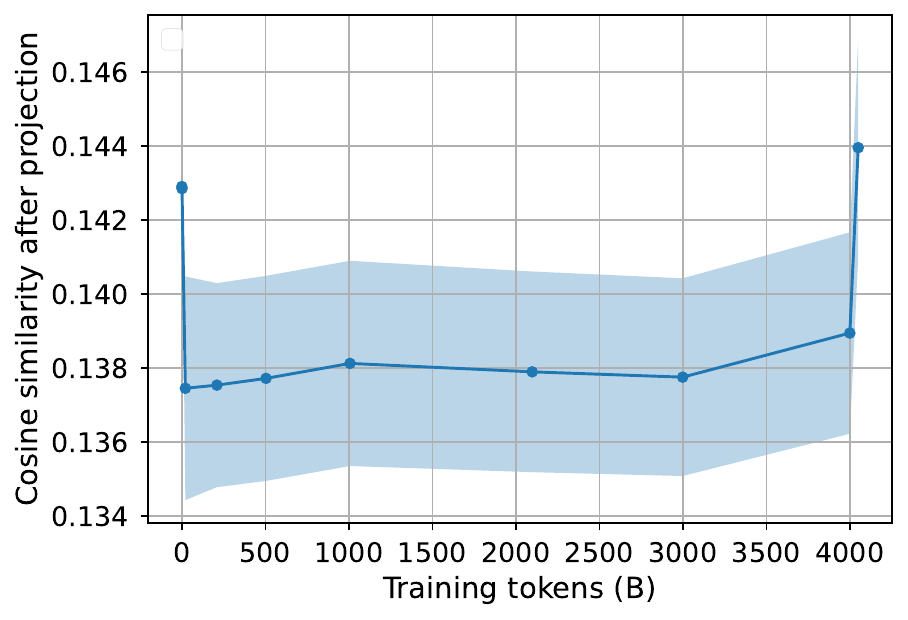}
        \caption{Cosine similarity after projection.}
    \end{subfigure}
    \caption{Gradient compression metrics along training for OLMo2-1B. Standard deviation is reported using filled areas.}
    \label{fig:training_dynamics}
\end{figure*}

\section{Downstream Evaluation -- Detailed Results}\label{app:bench_scores}

In \Cref{tab:benchmark_results}, we provide a detailed report of zero-shot performance of our rank-constrained models along training. We observe that, despite some benchmarks being more prone to variance effects, the overall trend favors models with higher values of $D$. We also report a graphical visualization in \Cref{fig:downstream_curves} to facilitate analysis.

\begin{table*}[t]
\centering \scriptsize
\begin{tabular}{c|ccccccc|cc}
\toprule
$D$ & ARC-C & ARC-E & HellaSwag & PIQA & SciQ & OpenBookQA & Lambada $\downarrow$ & Avg. & Avg. (weighted) \\
\midrule
\multicolumn{10}{c}{\textbf{5B tokens}} \\
\midrule
32 & 23.63 & 42.21 & 33.17 & 60.88 & 50.40 & 30.20 & 593.70 & 40.08 & 39.22 \\
64 & 25.77 & 45.50 & 34.57 & 62.62 & 60.50 & 31.80 & 186.90 & 43.46 & 41.46 \\
128 & 26.96 & 51.30 & 36.31 & 64.36 & 66.60 & \textbf{34.00} & 104.50 & 46.59 & 43.96 \\
256 & 28.24 & 52.23 & 37.39 & 64.58 & 69.30 & 33.40 & 70.99 & 47.52 & 44.93 \\
512 & 27.82 & 53.41 & 37.66 & 62.68 & 72.90 & 32.40 & 62.47 & 47.81 & 45.06 \\
1024 & 28.84 & 52.02 & 38.65 & 66.05 & 72.40 & 31.80 & 56.16 & 48.29 & 46.01 \\
2048 & 28.84 & \textbf{55.09} & 38.64 & 66.16 & \textbf{73.70} & 31.60 & 55.75 & \textbf{49.01} & \textbf{46.50} \\
4096 & \textbf{29.10} & 53.79 & \textbf{38.87} & \textbf{66.54} & 72.40 & 32.60 & \textbf{54.39} & 48.88 & 46.49 \\
\midrule
\multicolumn{10}{c}{\textbf{8.5B tokens}} \\
\midrule
32 & 26.11 & 43.43 & 34.99 & 60.72 & 53.20 & 32.20 & 496.49 & 41.77 & 40.74 \\
64 & 27.22 & 47.52 & 36.90 & 64.47 & 61.80 & 31.20 & 162.42 & 44.85 & 43.48 \\
128 & 26.54 & 52.06 & 37.63 & 64.74 & 67.50 & \textbf{35.80} & 85.05 & 47.38 & 44.92 \\
256 & 27.90 & 54.80 & 38.80 & 66.21 & 73.10 & 32.80 & 59.74 & 48.94 & 46.49 \\
512 & 27.73 & 52.95 & 39.65 & 66.21 & 72.90 & 34.60 & 52.77 & 49.01 & 46.75 \\
1024 & 28.41 & 55.39 & 40.90 & 66.32 & 73.40 & 32.40 & 45.46 & 49.47 & 47.79 \\
2048 & \textbf{31.40} & 55.85 & 41.09 & 66.87 & \textbf{74.10} & 33.60 & 45.50 & 50.48 & 48.31 \\
4096 & 29.69 & \textbf{56.10} & \textbf{41.78} & \textbf{67.36} & 72.40 & \textbf{35.80} & \textbf{45.23} & \textbf{50.52} & \textbf{48.67} \\
\midrule
\multicolumn{10}{c}{\textbf{11B tokens}} \\
\midrule
32 & 25.94 & 44.19 & 35.70 & 61.86 & 52.90 & 31.60 & 406.56 & 42.03 & 41.37 \\
64 & 28.58 & 47.52 & 37.69 & 64.09 & 61.00 & 32.20 & 129.72 & 45.18 & 43.92 \\
128 & 26.28 & 50.38 & 38.48 & 64.85 & 64.40 & 33.20 & 73.17 & 46.27 & 44.93 \\
256 & 28.75 & 54.92 & 39.81 & 67.03 & 71.70 & 32.20 & 47.94 & 49.07 & 47.17 \\
512 & 30.29 & 54.59 & 40.99 & 66.59 & 69.70 & 33.20 & 43.29 & 49.23 & 47.72 \\
1024 & 29.86 & 55.77 & 41.41 & 67.79 & \textbf{73.80} & 33.60 & 38.06 & 50.37 & 48.52 \\
2048 & 30.20 & 55.18 & 42.11 & 67.79 & \textbf{73.80} & 34.00 & 37.58 & 50.51 & 48.86 \\
4096 & \textbf{31.06} & \textbf{56.52} & \textbf{42.64} & \textbf{68.72} & 71.30 & \textbf{34.40} & \textbf{37.27} & \textbf{50.77} & \textbf{49.41} \\
\bottomrule
\end{tabular}
\caption{Benchmark scores grouped by training tokens (5B, 8.5B, 11B) and ordered by effective hidden dimension $D$. We report both weighted and unweighted averages.}
\label{tab:benchmark_results}
\end{table*}

\begin{figure*}[t!]
    \centering
    \begin{subfigure}[t]{0.4\textwidth}
        \centering
        \includegraphics[width=\textwidth]{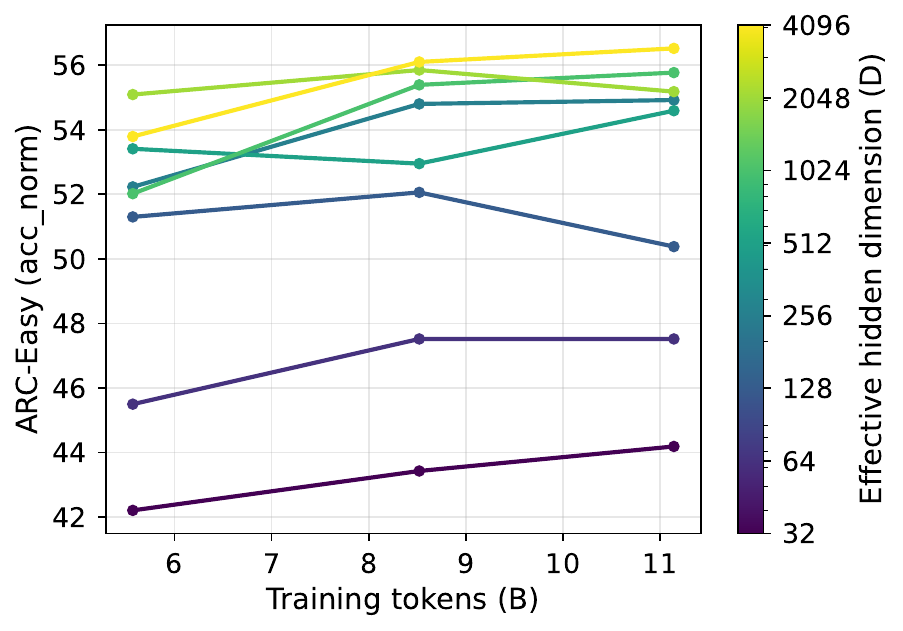}
        \caption{ARC-Easy}
    \end{subfigure}
    \begin{subfigure}[t]{0.4\textwidth}
        \centering
        \includegraphics[width=\textwidth]{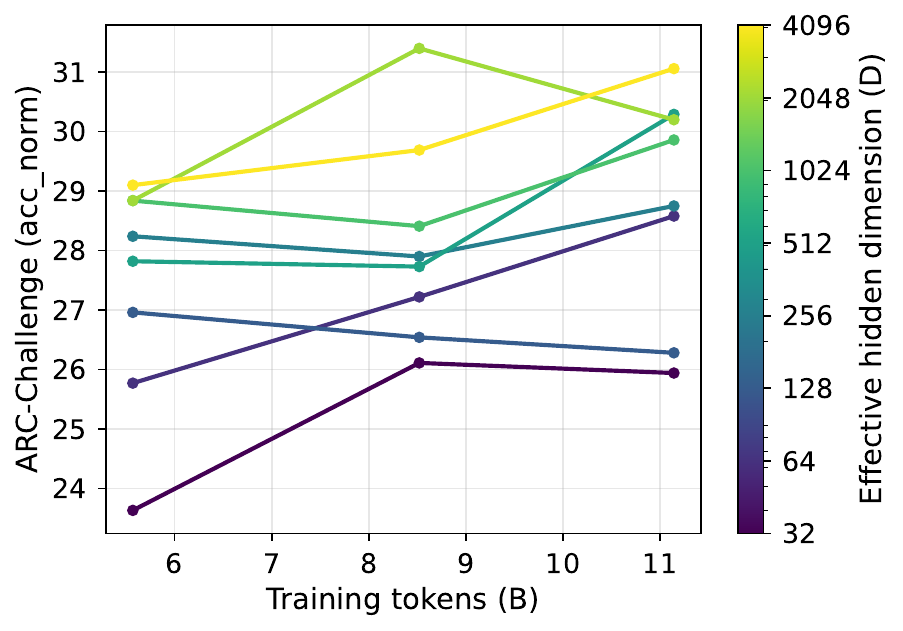}
        \caption{ARC-Challenge}
    \end{subfigure}
    \begin{subfigure}[t]{0.4\textwidth}
        \centering
        \includegraphics[width=\textwidth]{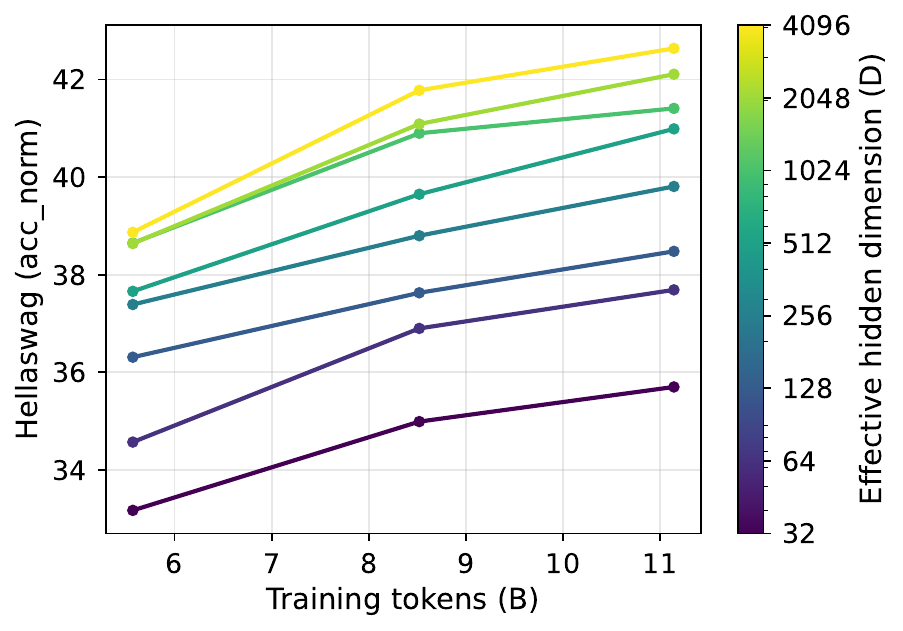}
        \caption{Hellaswag}
    \end{subfigure}
    \begin{subfigure}[t]{0.4\textwidth}
        \centering
        \includegraphics[width=\textwidth]{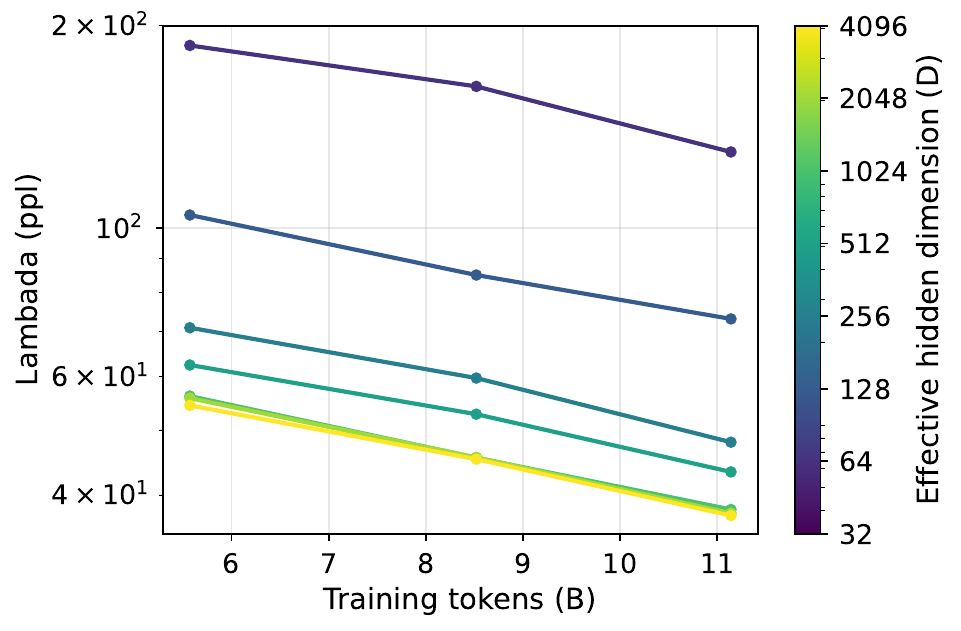}
        \caption{Lambada}
    \end{subfigure}
    \begin{subfigure}[t]{0.4\textwidth}
        \centering
        \includegraphics[width=\textwidth]{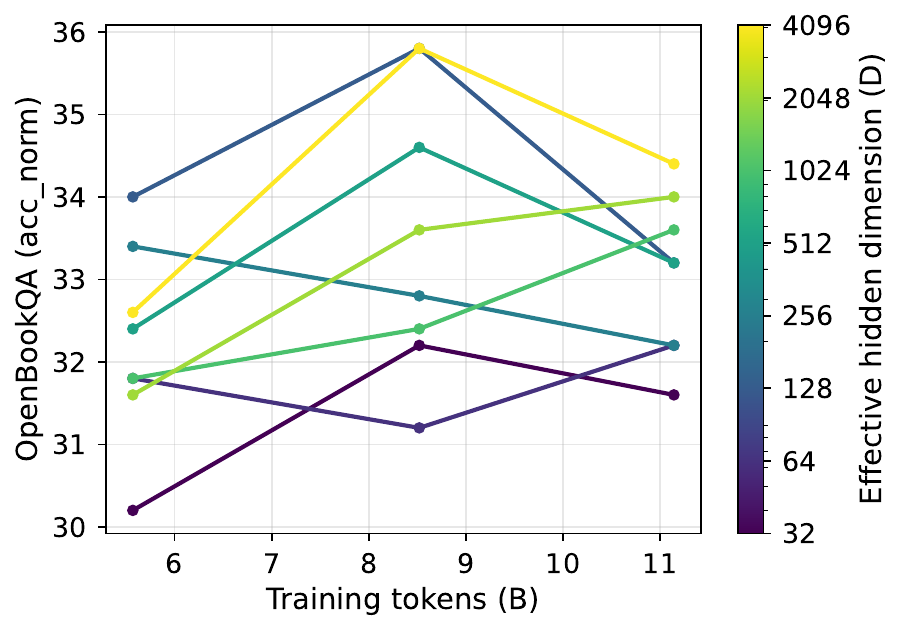}
        \caption{OpenBookQA}
    \end{subfigure}
    \begin{subfigure}[t]{0.4\textwidth}
        \centering
        \includegraphics[width=\textwidth]{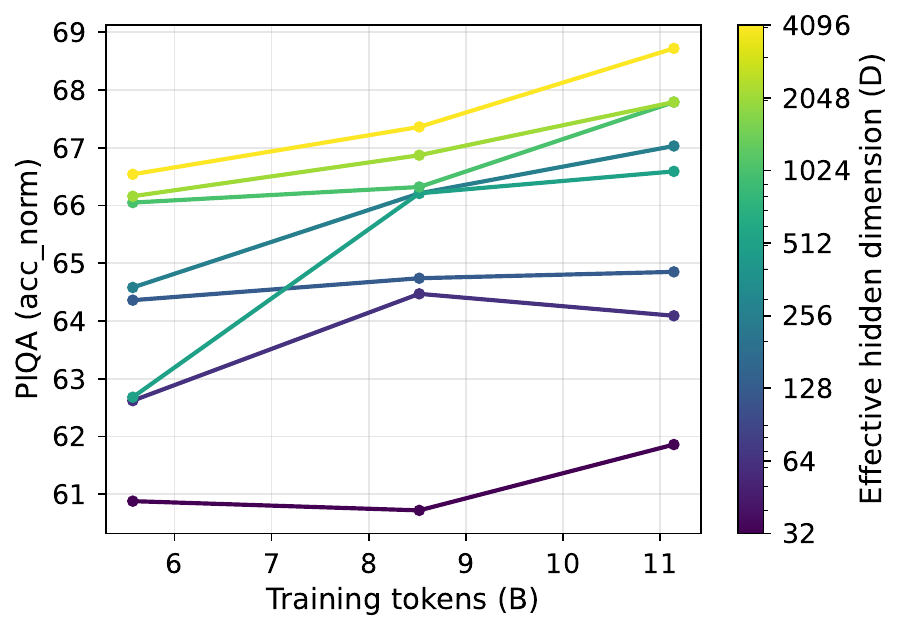}
        \caption{PIQA}
    \end{subfigure}
    \begin{subfigure}[t]{0.4\textwidth}
        \centering
        \includegraphics[width=\textwidth]{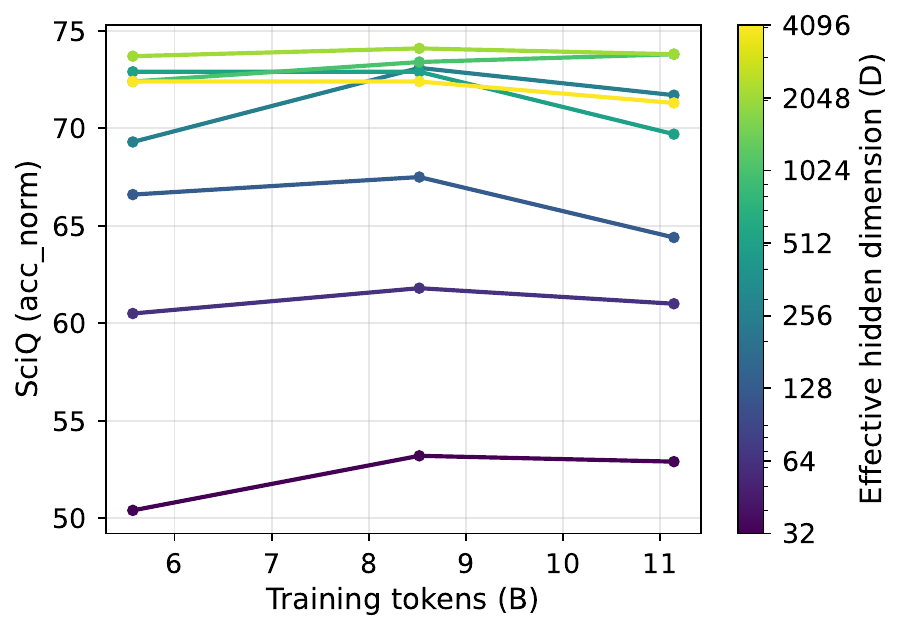}
        \caption{SciQ}
    \end{subfigure}%
    \caption{Downstream scores evolution along training for different benchmarks.}
    \label{fig:downstream_curves}
\end{figure*}

\FloatBarrier
\section{Experimental setups}
\label{app:exp_setups}

\subsection{\texttt{SpamLang} Experiments}
\label{app:spamlang_hyperparams}
We report hyperparameters for \texttt{SpamLang} experiments in \Cref{tab:spamlang_hyperparams}. Additionally, our initial experimentations on the impact of weight decay on the training dynamics showed minimal impact. We ran experiments for $D=32,768$ using a learning rate of 1e-4, and obtain final losses within the $[0.2, 0.4]$ range.  

\begin{table}[tbh]
\centering
\begin{tabular}{ll}
\toprule
\textbf{Hyperparameter} & \textbf{Value} \\
\midrule
Num. Layers & 30 \\
Head Dimension & 64 \\
Num. Attention Heads & 9 \\
Num. KV Heads & 3 \\
Hidden size & \num{576} \\
Intermediate size & \num{1536} \\
\midrule
Precision & Mixed (\texttt{bfloat16}) \\
Sequence Length & 64 \\
Batch Size & 128 \\
Total Training Steps & 5000 \\
Warmup Steps & 500 \\
Weight Decay & 0.01 \\
Optimizer & AdamW \\
Adam $\beta_1, \beta_2$ & (0.9, 0.95) \\
Adam $\epsilon$ & $1 \times 10^{-8}$ \\
Learning Rate Schedule & Cosine Decay \\
Gradient Clipping & 1.0 \\
\bottomrule
\end{tabular}
\caption{Training hyperparameters used for the \texttt{SpamLang} experiment. Experiments were run on single A6000 NVIDIA GPUs.}
\label{tab:spamlang_hyperparams}
\end{table}

Additionally, we provide a side-by-side comparison of our learning rate and vocabulary size sweeps in \Cref{fig:spamlang_final_comparision}.

\begin{figure*}[tbh!]
    \centering
    \begin{subfigure}[t]{0.45\textwidth}
        \centering
        \includegraphics[width=\textwidth]{imgs/loss_heatmap_tied.pdf}
        \caption{Tied embeddings}
    \end{subfigure}%
    ~ 
    \begin{subfigure}[t]{0.45\textwidth}
        \centering
        \includegraphics[width=\textwidth]{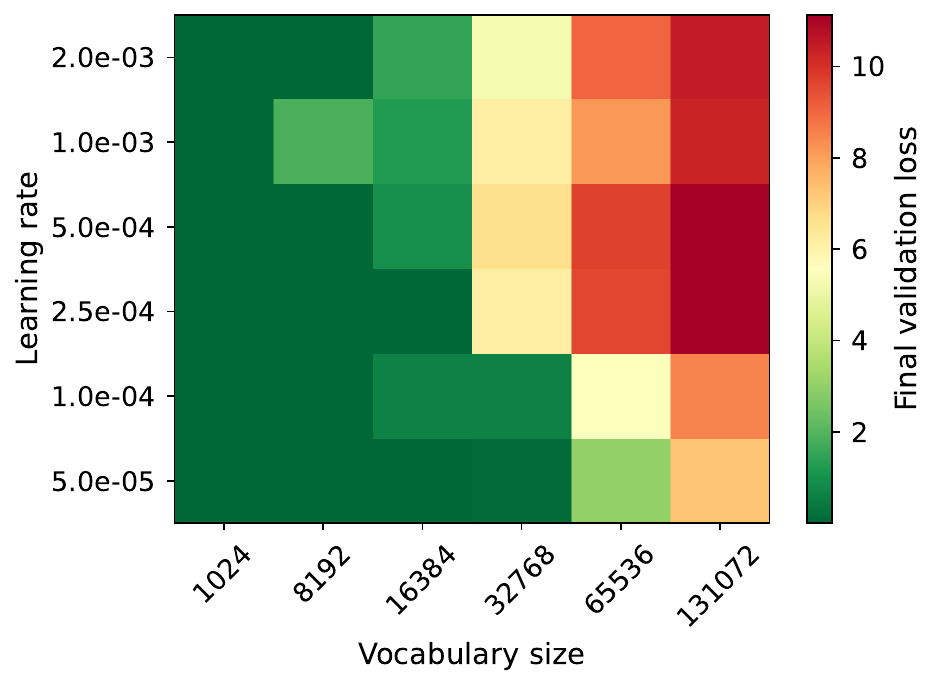}
        \caption{Untied embeddings}
    \end{subfigure}
    \caption{Final validation loss for 106M (non-embedding) parameter Transformer models trained on the \texttt{SpamLang} synthetic language for various vocabulary sizes and learning rates. We report untied and tied embedding setups, and observe no strong effect from this hyper-parameter choice. The hidden dimension is set to 576 for all models.}
    \label{fig:spamlang_final_comparision}
\end{figure*}

\FloatBarrier

\subsection{Pretraining Experiments}
\label{app:pretrain_hyperparams}

We report hyperparameters for our pretraining experiments in \Cref{tab:gpt_bottleneck_hyperparams}. We perform the decay phase of the WSD schedule at 3 intermediate steps:
\begin{itemize}
    \item After $\num{80000}$ steps, where we decay for 5000 steps;
    \item After $\num{120000}$ steps, where we decay for 10000 steps;
    \item After $\num{170000}$ steps, where we decay for 10000 steps.
\end{itemize}
We adjust the decay length to remain within a 5-10\% ratio of the total training steps. For every model, the decay phase is performed on the same data as the stable phase. We observe considerable performance gains during the decay phase, as reported in~\citet{hu2024minicpm}.

\begin{table}[tbh]
\centering
\begin{tabular}{lc}
\toprule
\textbf{Hyperparameter} & \textbf{Value} \\
\midrule
Num. Layers & 6 \\
Head Dimension & 128 \\
Num. Attention Heads & 32 \\
Hidden size & \num{4096} \\
Intermediate size & \num{16384} \\
Weight Tying & No \\
\midrule
Precision & Mixed (\texttt{bfloat16}) \\
Sequence Length & 512 \\
Train Batch Size & 128 \\
Total Training Steps & \num{170000} \\
Learning Rate & $3 \times 10^{-4}$ \\
Warmup Steps & \num{2000} \\
Cooldown Steps & \num{10000} \\
Weight Decay & 0.1 \\
Optimizer & AdamW \\
Adam $\beta_1, \beta_2$ & (0.9, 0.95) \\
Adam $\epsilon$ & $1 \times 10^{-8}$ \\
LR Schedule & Warm-up Stable Decay (WSD) \\
Decay Type & Cosine (5-10\%) \\
Gradient Clipping & 1.0 \\
\bottomrule
\end{tabular}
\caption{Training hyperparameters used for the pretraining experiments with constrained heads. Experiments were run on B200 GPUs, totaling $\sim$760 hours.}
\label{tab:gpt_bottleneck_hyperparams}
\end{table}

\section{Gradient Compression Example}\label{app:gradcomp_ex}

In \Cref{fig:cp_wayne,fig:cp_may,fig:cp_1907}, we display the strongest coefficients (in absolute value) of the gradients projected on the null space of the LM head for target tokens ``\textit{Wayne}'', ``\textit{May}'' and ``\textit{1907}'' in the sentence ``\textit{John Wayne was born on May 26th, 1907.}''. We also report the logit gradient coefficients as a comparison point. The models are taken from the GPT-2 suite~\citep{radford2019language}.
\begin{figure*}[t!]
    \centering
    \begin{subfigure}[t]{0.9\textwidth}
        \centering
        \includegraphics[width=\textwidth]{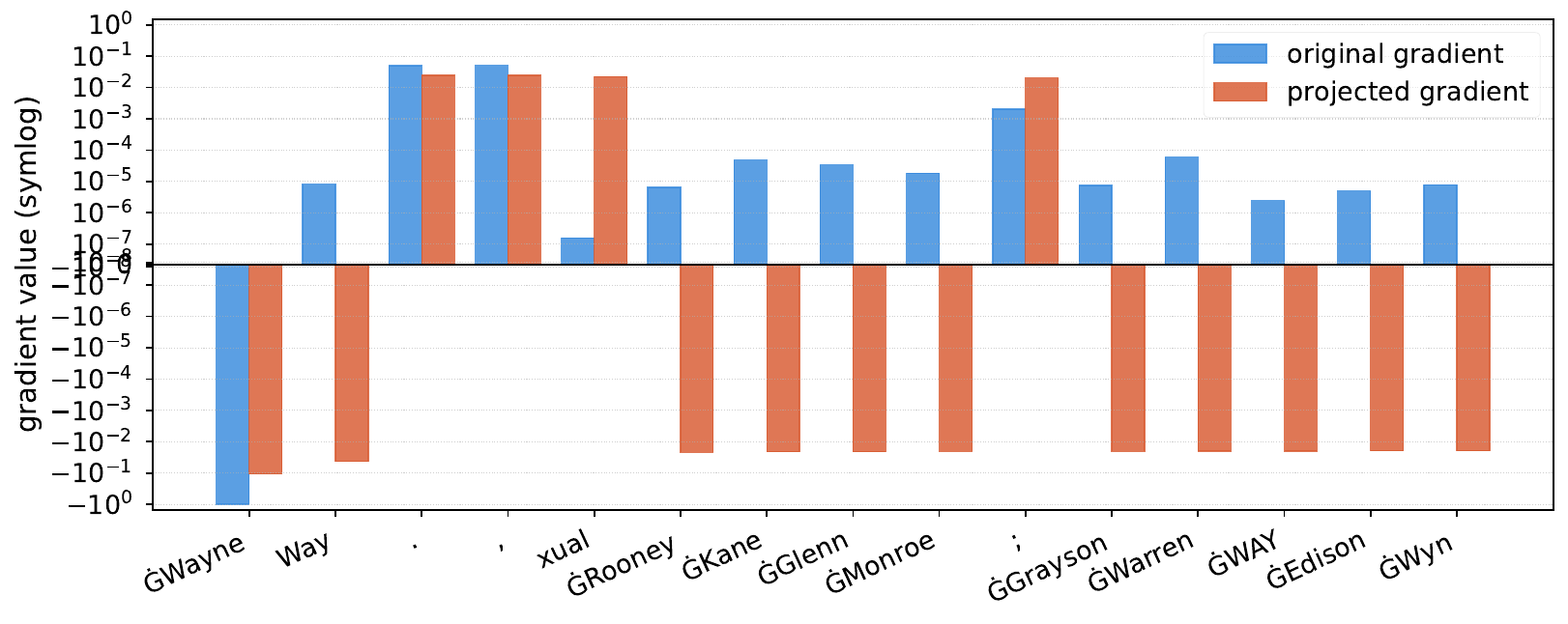}
        \caption{GPT2-Small}
    \end{subfigure}
    \begin{subfigure}[t]{0.9\textwidth}
        \centering
        \includegraphics[width=\textwidth]{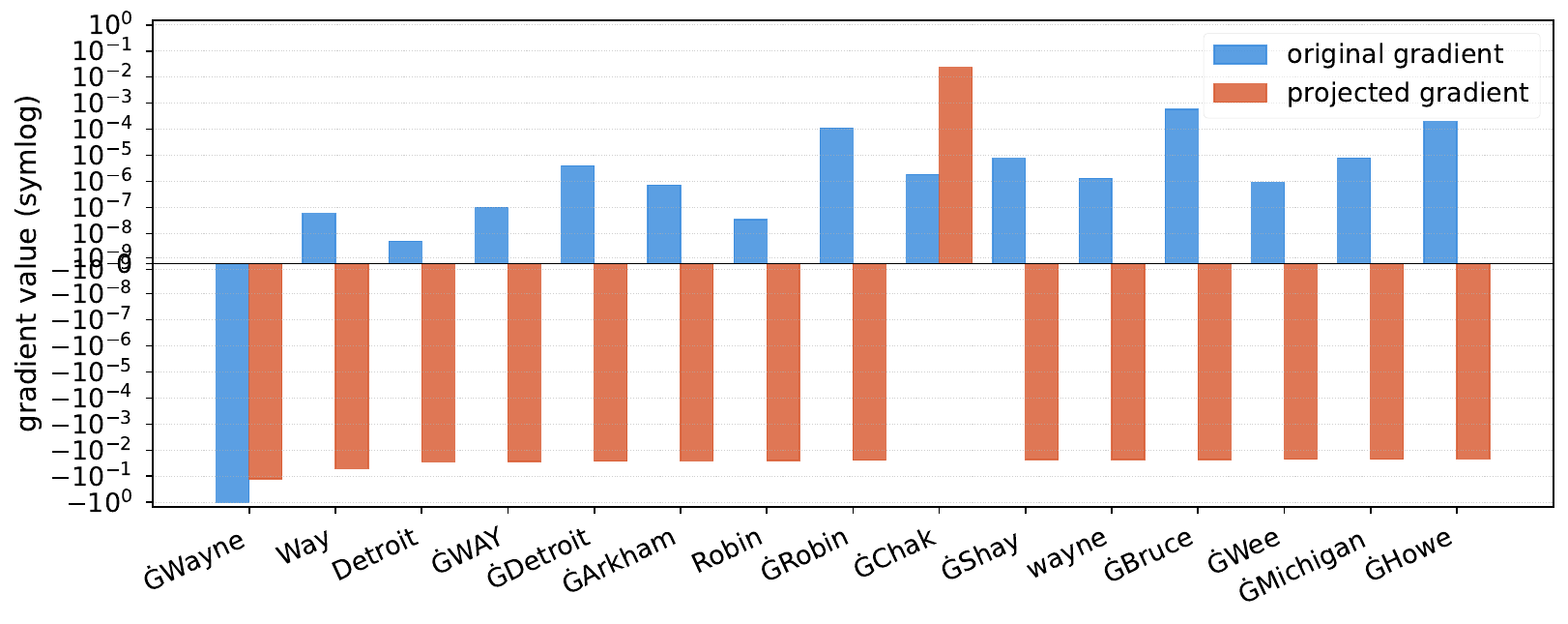}
        \caption{GPT2-Medium}
    \end{subfigure}
    \begin{subfigure}[t]{0.9\textwidth}
        \centering
        \includegraphics[width=\textwidth]{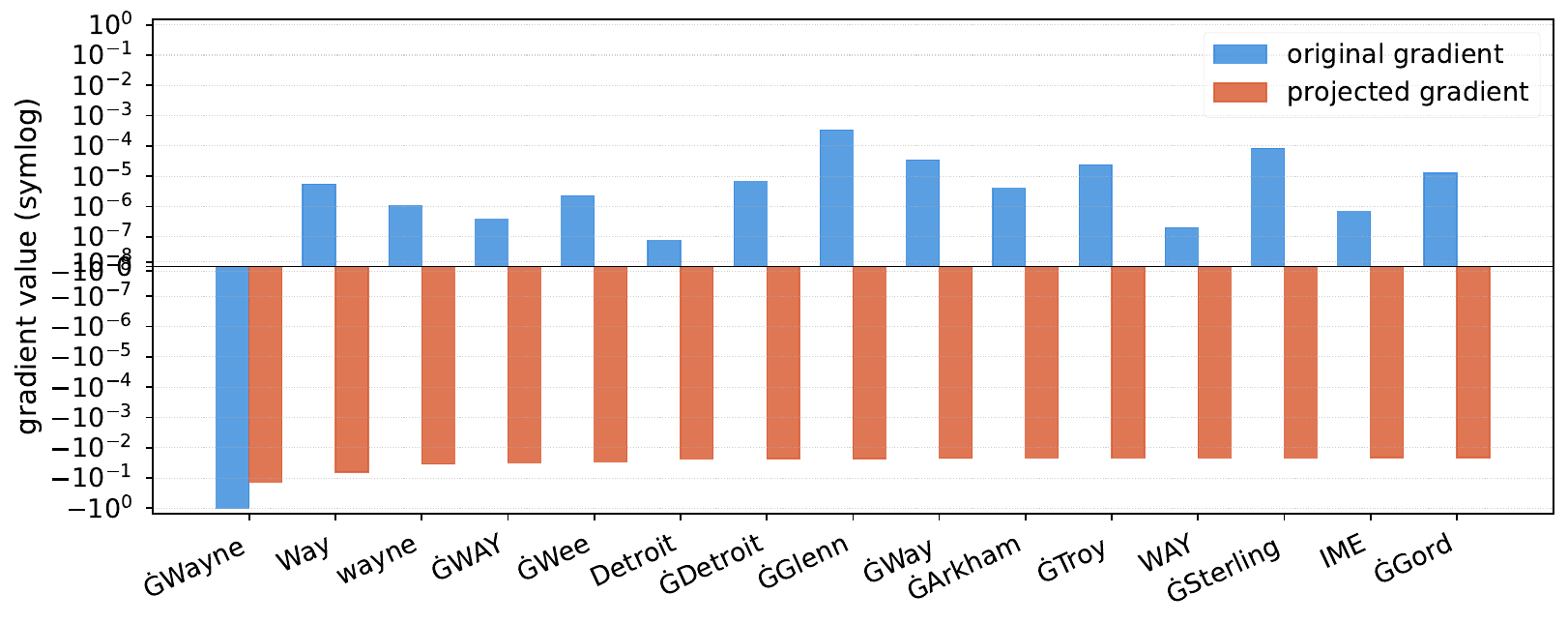}
        \caption{GPT2-Large}
    \end{subfigure}
    \begin{subfigure}[t]{0.9\textwidth}
        \centering
        \includegraphics[width=\textwidth]{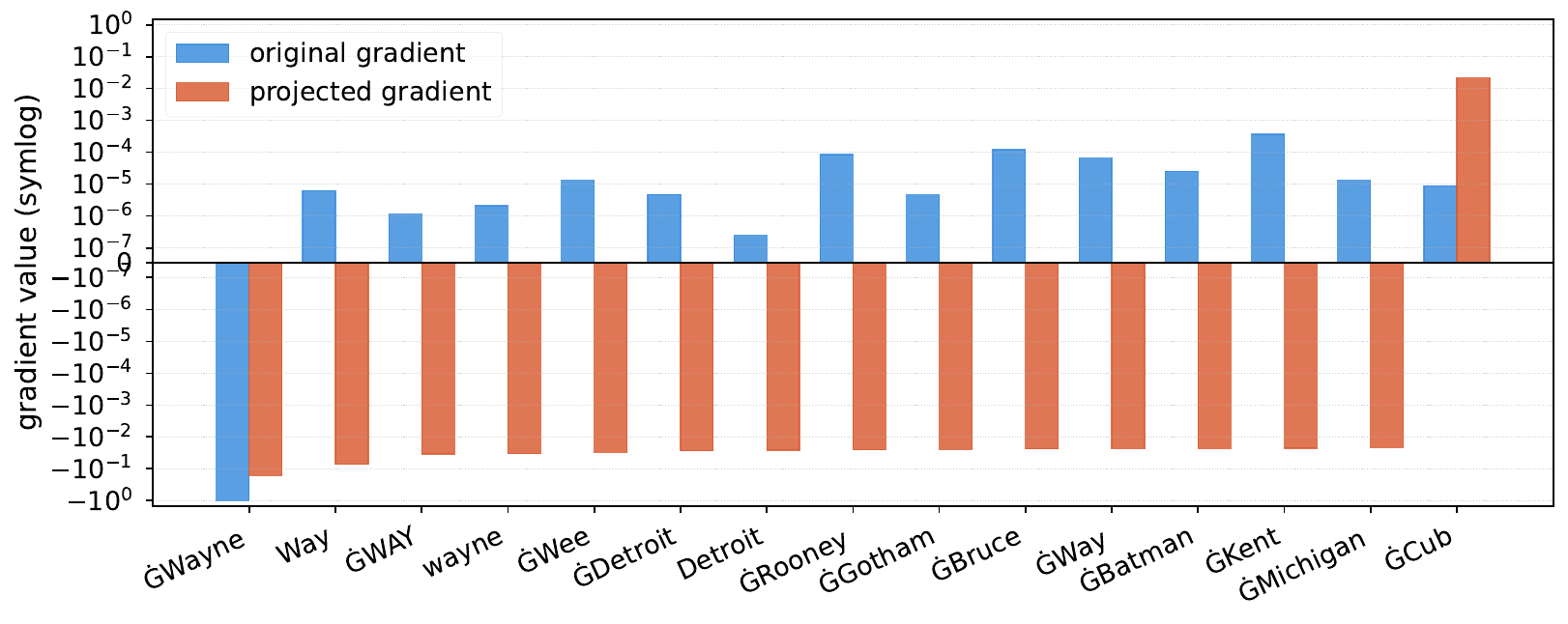}
        \caption{GPT2-XL}
    \end{subfigure}
    \caption{Most important gradient coefficients ranked by post-projection highest absolute values (normalized). The sentence is ``\textit{John \textbf{Wayne} was born on May 26th, 1907.}'', where the target token for this specific gradient is bolded.}
    \label{fig:cp_wayne}
\end{figure*}

\begin{figure*}[t!]
    \centering
    \begin{subfigure}[t]{0.9\textwidth}
        \centering
        \includegraphics[width=\textwidth]{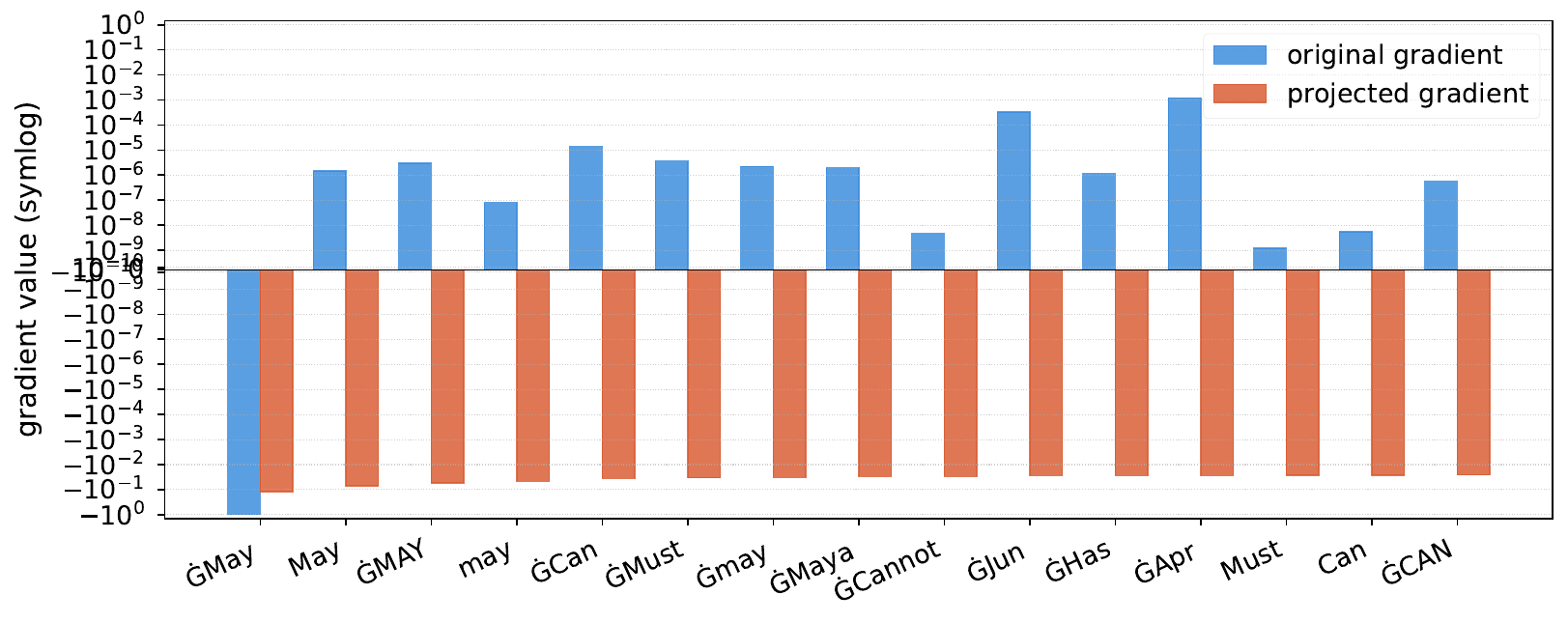}
        \caption{GPT2-Small}
    \end{subfigure}
    \begin{subfigure}[t]{0.9\textwidth}
        \centering
        \includegraphics[width=\textwidth]{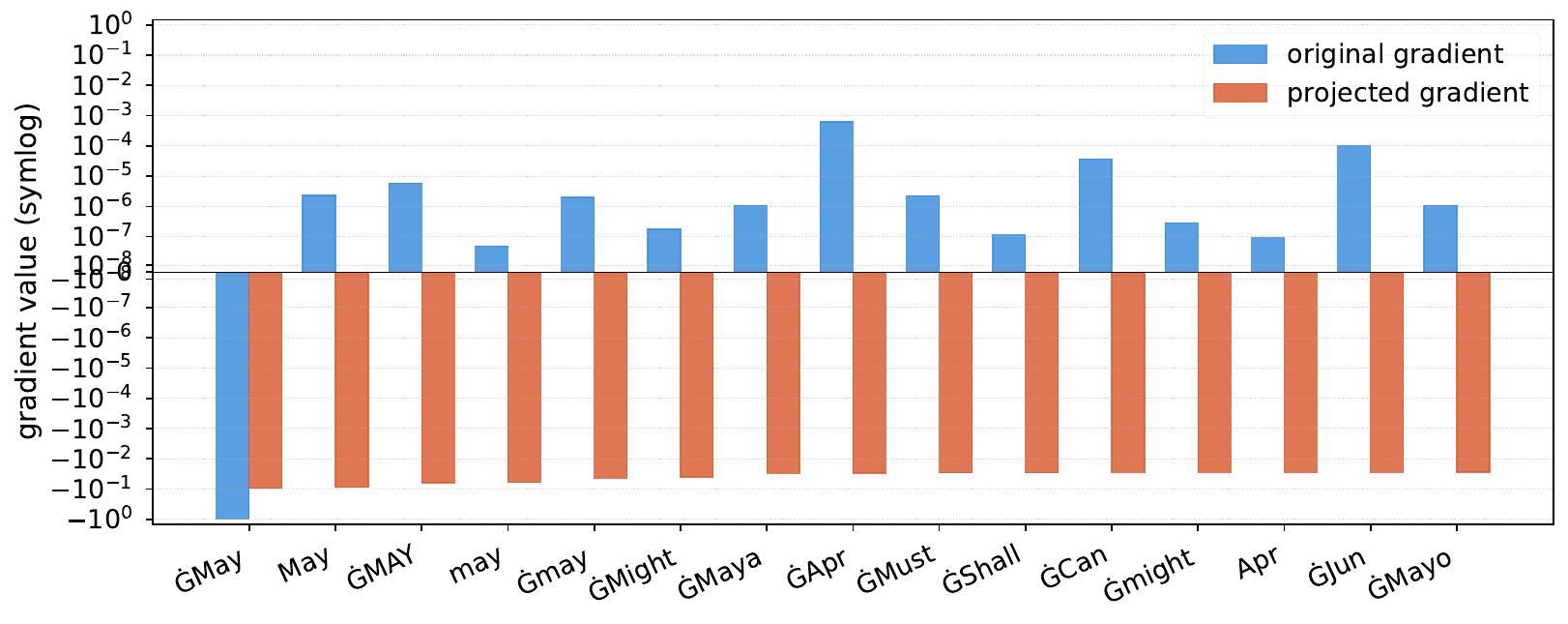}
        \caption{GPT2-Medium}
    \end{subfigure}
    \begin{subfigure}[t]{0.9\textwidth}
        \centering
        \includegraphics[width=\textwidth]{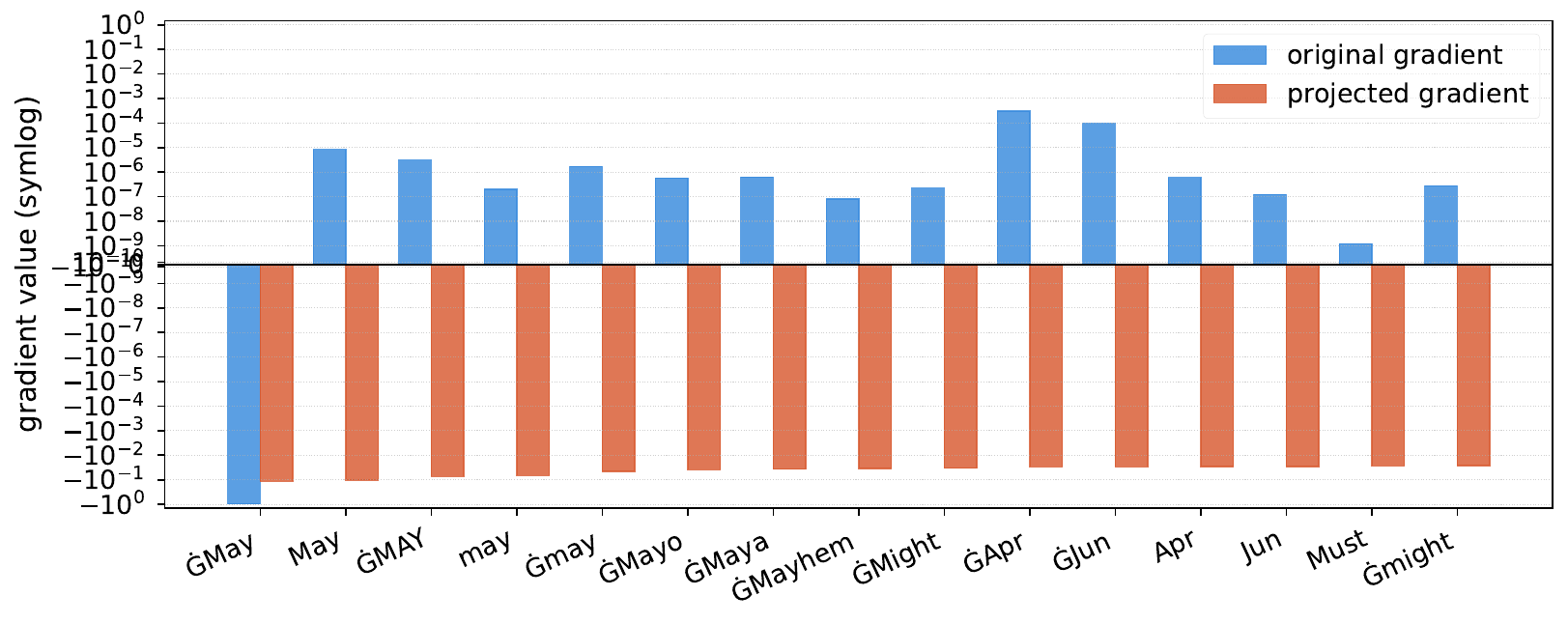}
        \caption{GPT2-Large}
    \end{subfigure}
    \begin{subfigure}[t]{0.9\textwidth}
        \centering
        \includegraphics[width=\textwidth]{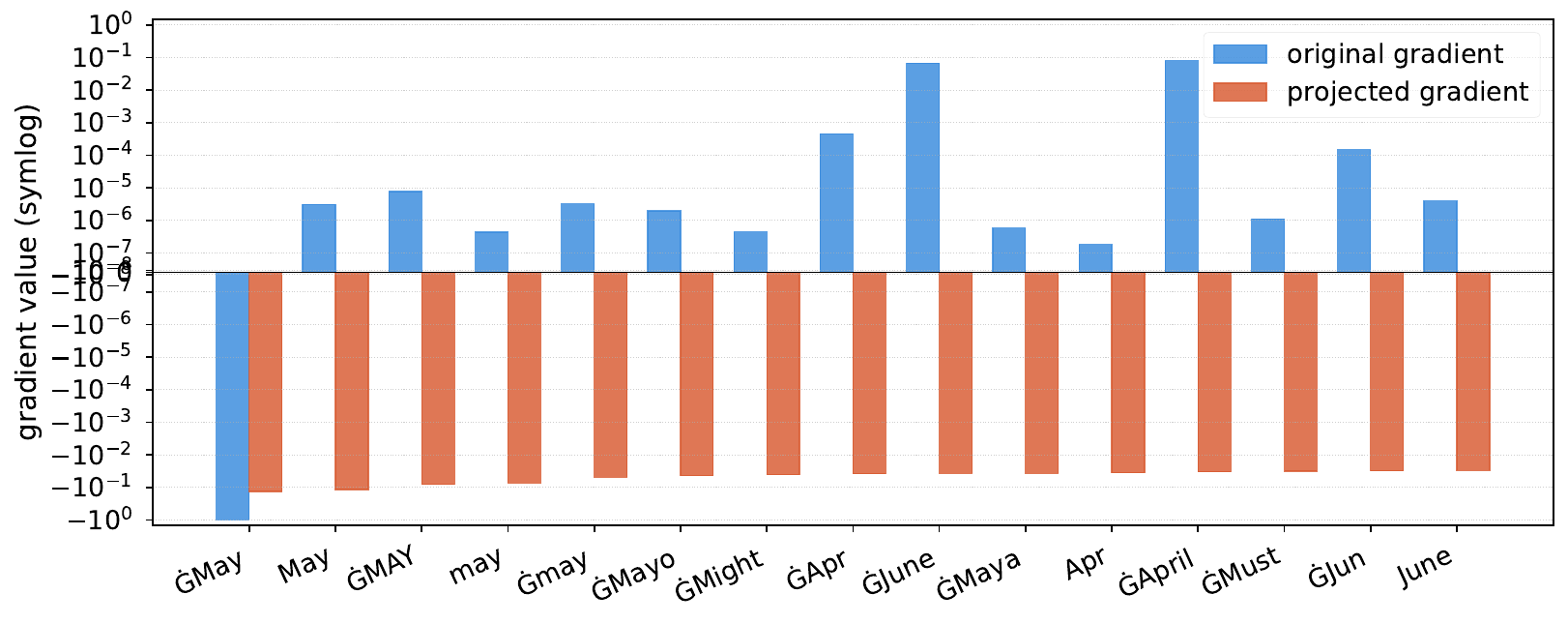}
        \caption{GPT2-XL}
    \end{subfigure}
    \caption{Most important gradient coefficients ranked by post-projection highest absolute values (normalized). The sentence is \textit{John Wayne was born on \textbf{May} 26th, 1907.}, where the target token for this specific gradient is bolded.}
    \label{fig:cp_may}
\end{figure*}

\begin{figure*}[t!]
    \centering
    \begin{subfigure}[t]{0.9\textwidth}
        \centering
        \includegraphics[width=\textwidth]{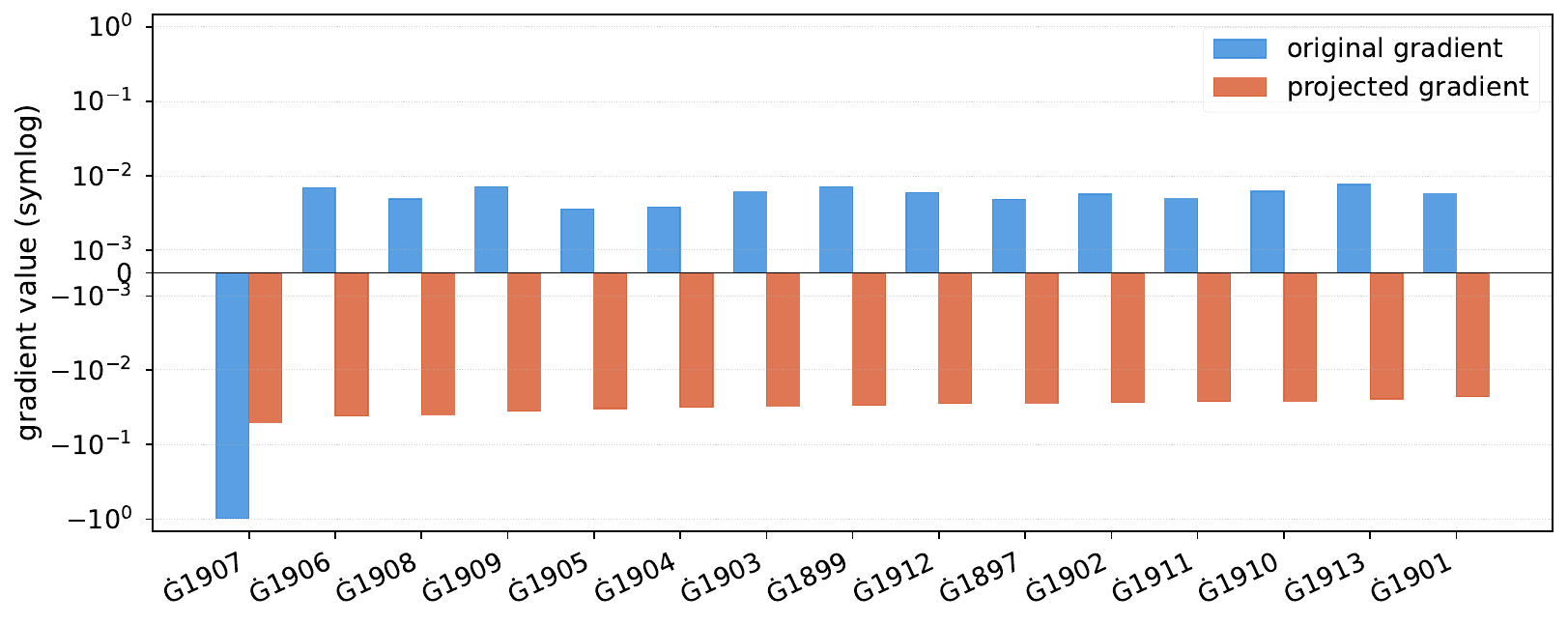}
        \caption{GPT2-Small}
    \end{subfigure}
    \begin{subfigure}[t]{0.9\textwidth}
        \centering
        \includegraphics[width=\textwidth]{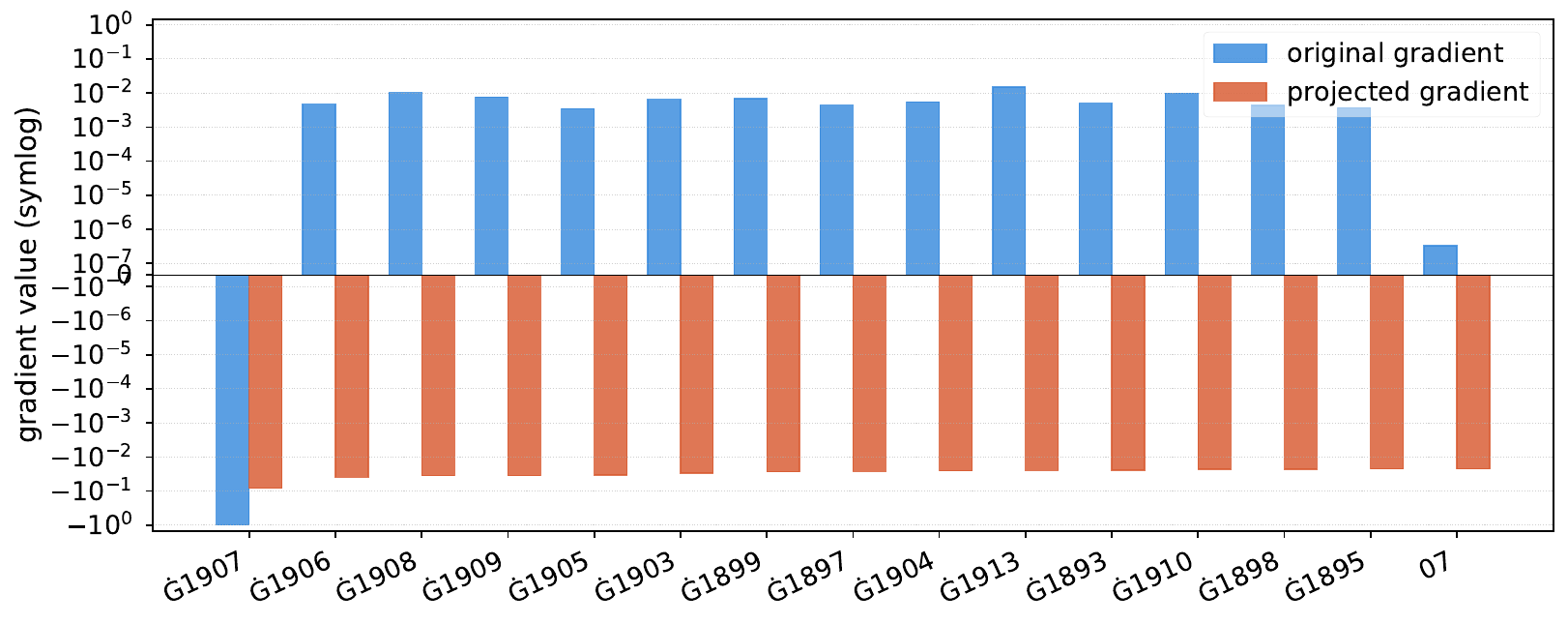}
        \caption{GPT2-Medium}
    \end{subfigure}
    \begin{subfigure}[t]{0.9\textwidth}
        \centering
        \includegraphics[width=\textwidth]{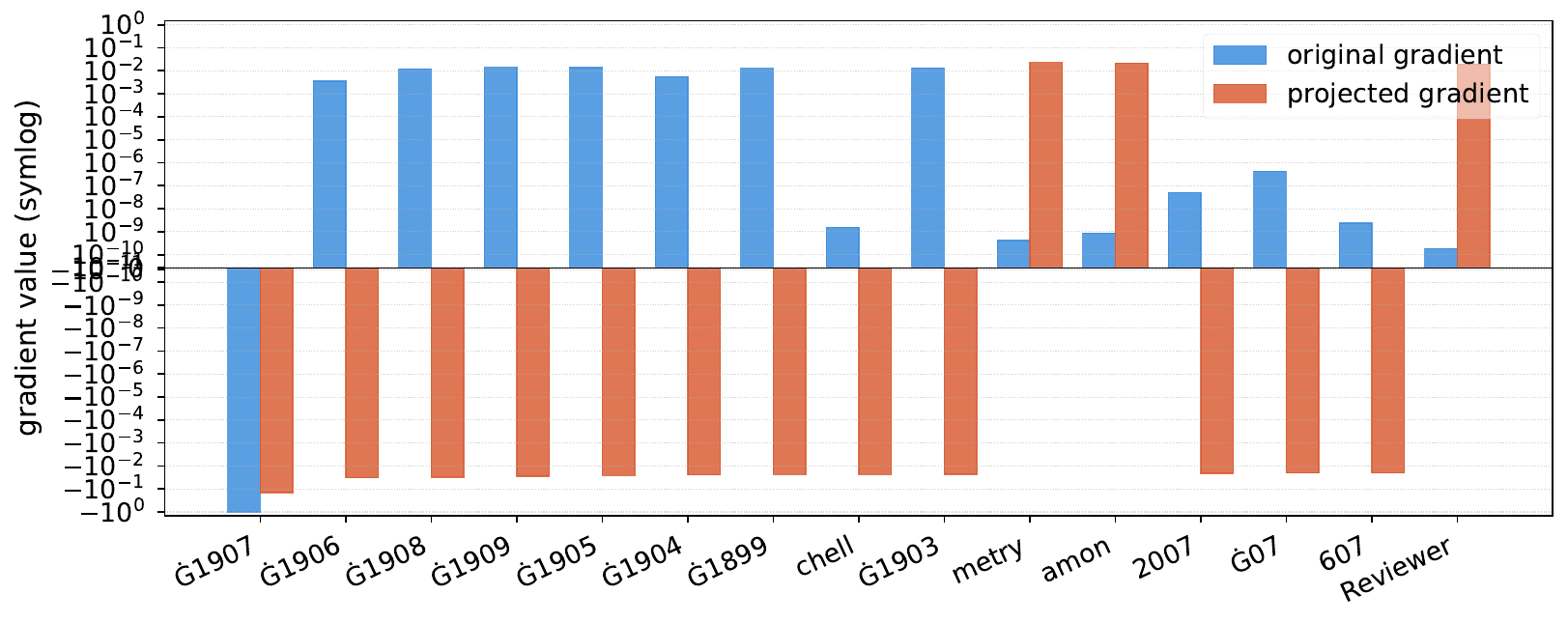}
        \caption{GPT2-Large}
    \end{subfigure}
    \begin{subfigure}[t]{0.9\textwidth}
        \centering
        \includegraphics[width=\textwidth]{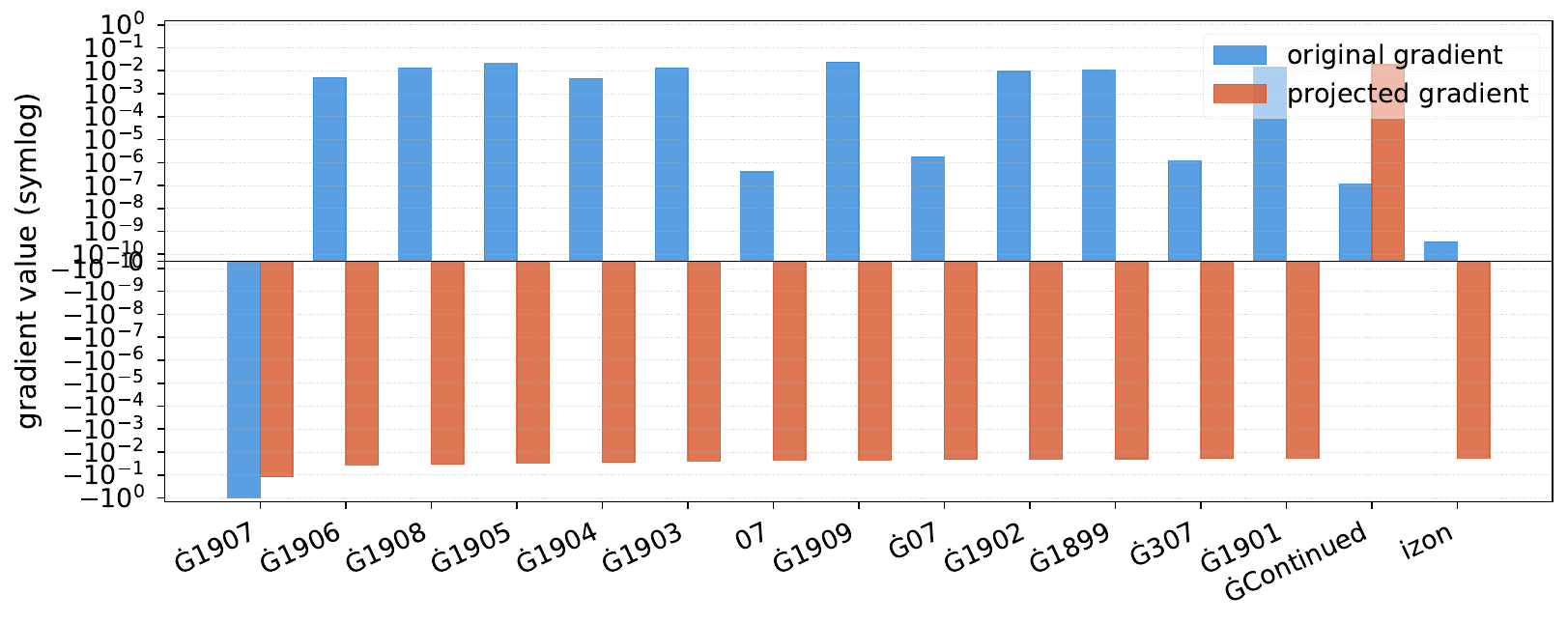}
        \caption{GPT2-XL}
    \end{subfigure}
    \caption{Most important gradient coefficients ranked by post-projection highest absolute values (normalized). The sentence is \textit{John Wayne was born on May 26th, \textbf{1907}.}, where the target token for this specific gradient is bolded.}
    \label{fig:cp_1907}
\end{figure*}
\FloatBarrier
\Cref{fig:cp_wayne,fig:cp_may,fig:cp_1907} show very distinct patterns that let us better understand how the bottleneck affects the shape of the training feedback provided to the backbone. Overall, we observe that in all cases, the projected gradient is poorly aligned with the original gradient. Even though the strongest coefficient always coincides with the observed ground-truth token, we notice a clear distortion of the signal for the rest of the tokens. We do not observe strong high-level differences across model sizes for these specific examples.

In \Cref{fig:cp_wayne}, where ``\textit{Wayne}'' is the target token coming after ``\textit{John}'', we notice that tokens that are differently related to the form ``\textit{Wayne}'' are associated to a non-negligible negative projected gradient coefficient. Notably, part of the training feedback is reinforcing tokens like ``\textit{Rooney}'' (as in the football player Wayne Rooney), ``\textit{Bruce}''/``\textit{Robin}''/``\textit{Arkham}'' (as in the Batman comics), or ``\textit{Detroit}''/``\textit{Michigan}'' (Wayne is a suburb of Detroit, Michigan (United States)). Across all model sizes, we also notice that subtokens and/or script variants of ``\textit{Wayne}'' carry almost as much gradient as the target token, which is reminiscent of~\citet{finlayson2024closing}. 

In \Cref{fig:cp_may}, where ``\textit{May}'' is the target token coming after ``\textit{John Wayne was born on}'', we notice a similar effect, where gradient coefficients are scattered across typographical variants of ``\textit{May}'', and other months which would be semantically valid in this context. We also notice that tokens that are semantically close to the verb ``\textit{May}'', such as ``\textit{Shall}'', ``\textit{Can}'', ``\textit{Might}'', or ``\textit{Must}'', are reinforced. Finally, tokens that are only lexically close to ``May'', such as ``Mayhem'' or ``Maya'', also absorb part of the gradient.

In \Cref{fig:cp_1907}, where ``\textit{1907}'' is the target token coming after ``\textit{John Wayne was born on May 26th,}'', we observe that the tokens with the highest gradient values besides ``\textit{1907}'' correspond to other years that are close to 1907. We also observe that some tokens that seem unrelated to 1907 appear in the top values, such as ``\textit{metry}'', ``\textit{amon}'', or ``\textit{Continued}''.

In conclusion, the weight of the top component of the original gradient is transferred to tokens that are generally related to the target token. However, this entanglement of the training signal is not very meaningful as it seems to be non-contextual and mixes different meaning of the same word form (e.g. \textit{May} $\rightarrow$ \textit{Might}). It also hampers the clarity of the training signal: it dilutes precise feedback (e.g. 1907) into broader categories (e.g. years surrounding 1907) or noise (e.g. ``\textit{metry}'').

\end{document}